%% file: acl_latex.tex
\documentclass[11pt]{article}
\usepackage{acl}
\usepackage{times}
\usepackage{latexsym}
\usepackage[T1]{fontenc}
\usepackage[utf8]{inputenc}
\usepackage{amsthm}
\theoremstyle{definition}

\usepackage[
  protrusion=true,
  expansion=false,
  tracking=false,
  kerning=false,
  spacing=false
]{microtype}
\usepackage{booktabs}
\usepackage{graphicx}
\usepackage{booktabs}
\usepackage{multirow}
\usepackage{tabularx}
\usepackage{array}
\usepackage{makecell}
\usepackage{adjustbox}
\usepackage{amsmath,amsfonts,amssymb,amsthm}
\usepackage{stfloats}
\usepackage[table]{xcolor}
\usepackage[most]{tcolorbox}
\tcbuselibrary{skins}
\usepackage{enumitem}
\usepackage{placeins}
\usepackage{float}
\usepackage{cuted}
\usepackage{caption}
\usepackage{listings}
\usepackage{algorithm}
\usepackage{algpseudocode}
\usepackage{calc}
\usepackage{kotex}
\usepackage{amsthm}
\newtheorem{theorem}{Theorem}
\newtheorem{definition}{Definition}
\usepackage{graphicx}
\usepackage[table]{xcolor}
\definecolor{boxbg}{RGB}{240,240,240}
\definecolor{softpurple}{HTML}{EFE9FF}
\definecolor{softgreen}{HTML}{E8F8EA}
\definecolor{scoregreen}{HTML}{2E9E45}
\definecolor{softred}{RGB}{255,90,90}
\definecolor{softblue}{RGB}{31,111,190}

\newcommand{\methodname}{DeRO}

\title{Negation without Confusion: Decomposing Bias and Recomposing Style \\ in Text Embeddings for Negation-aware Image Generation}

\author{Sooyeon Park$^{\spadesuit}$ \quad
Hemraj Singh$^{\spadesuit}$ \quad
Kyunghyun Min$^{\spadesuit}$ \quad
Youngjae Yu$^{\diamond}$ \quad
Sung-Bae Cho$^{\spadesuit}$ \\[0.4em]
$^{\spadesuit}$ Yonsei University \qquad
$^{\diamond}$ Seoul National University \\
\texttt{\{dreamyou070, shemraj07, khmin1104, sbcho\}@yonsei.ac.kr} \\
\qquad
\texttt{youngjaeyu@snu.ac.kr}
}

\definecolor{boxframe}{RGB}{0,0,0}
\begin{document}
\maketitle
\input{section/1_main}

\bibliography{custom}
\appendix
\input{section/2_supplementary}
\end{document}

%% file: section/1_main.tex
\begin{abstract}
Negation-aware image generation remains challenging for text-to-image diffusion models, which often struggle with negation due to implicit linguistic biases that are deeply entangled within text representations. As a result, bias-aligned semantics frequently dominate over intended negative constraints. To overcome this limitation, we propose a simple training-free method of \textbf{De}compose and \textbf{R}ec\textbf{O}mpose text embedding (\methodname{}), which constructs a semantic-preserving subspace and a bias subspace using auxiliary prompts and performs semantic disentanglement based on bidirectional SVD. By decomposing entangled representations and selectively recomposing disentangled semantic components, \methodname{} realigns text representations toward intended negation semantics while restoring residual stylistic features to preserve semantic fidelity and image quality. As a model-agnostic inference-time framework operating solely on text representations, \methodname{} can be seamlessly applied to diverse diffusion architectures, including both U-Net and DiT based models. Extensive experiments on both object- and concept-level negation demonstrate consistent improvements in negation-aware image generation, achieving stronger text-image alignment, better semantic fidelity, and higher human preference scores.
\end{abstract}

\begin{figure}[h]
\centering
\includegraphics[width=\linewidth]{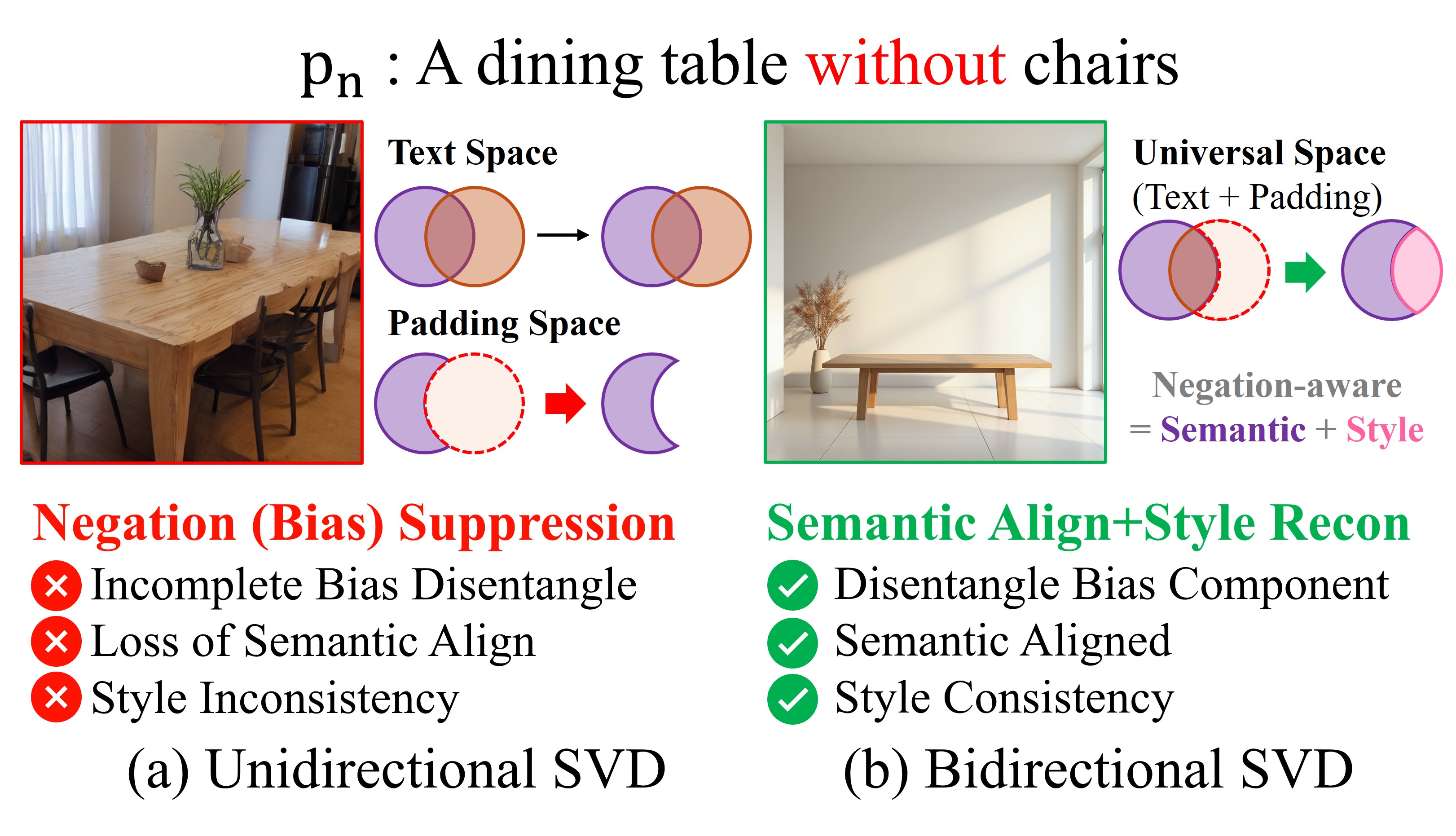}
\caption{(a) Existing \textbf{unidirectional SVD} methods (e.g., SupEOT~\cite{li2023get}) suppress only bias or negation directions within padding representations in order to remove undesired information while preserving the original prompt style. However, due to the strong entanglement among semantic, stylistic, and bias components, residual biased semantics often remain in the text embedding, resulting in incomplete negation or degraded stylistic consistency. (b) In contrast, \textbf{bidirectional SVD} operates in a universal embedding, enabling precise removal of bias components while preserving semantic and stylistic information, thereby achieving faithful negation-aware generation.}
\label{figure_1}
\end{figure}

\section{Introduction} \label{section_1}
Recent Diffusion Transformer (DiT) based text-to-image (T2I) models have achieved remarkable progress in visual quality and semantic alignment capabilities ~\citep{flux2024,podell2024sdxl}, but they still struggle to reliably handle negation prompts~\cite{li2023get,chen2026normalized}.  For example, even for a prompt of ``a bakery without bread,'' the undesired concept ``bread'' is often still generated as shown in Fig.~\ref{figure_2}.

% [2] 신규: 부정 프롬프트의 특징 → entanglement 프레이밍
A negation prompt inherently contains two coupled semantic components: contextual semantics to be \emph{preserved} and target attributes to be \emph{removed}. The difficulty of negation largely depends on the semantic bias between these components. When weakly correlated (e.g., \emph{room} and \emph{bread}), the negated concept is reliably removed; however, strongly coupled pairs, such as \emph{bakery} and \emph{bread}, often persist in generated images. This mainly arises from implicit linguistic biases formed during large-scale training~\cite{geirhos2020shortcut,park2026multi}, which strongly entangle certain semantics within text representations~\cite{fu2026fairimagen,friedrich2023fair}. Therefore, negation in T2I models is fundamentally a problem of disentangling preserved and removed semantics rather than simply suppressing target concepts.

% [3] 기존 방법의 한계 % 
Existing negation-aware generation methods mainly focus on suppressing unwanted semantics through attention manipulation~\cite{chen2026normalized} or direct embedding suppression~\cite{li2023get}. However, as illustrated in Fig.~\ref{figure_1}(a), suppression-only methods struggle under strong semantic coupling~\cite{saha2025side}, consequently suffering from semantic collapse, style distortion, or incomplete negation.

% % [4] 제안 방법 및 핵심 아이디어 (통합 및 축약)
To address these limitations, we propose \methodname{} (\textbf{De}compose and \textbf{R}ec\textbf{O}mpose text embedding), a simple yet effective method that treats negation failure as representation-level bias entanglement. Unlike suppression based methods\cite{li2023get,biswas2025cure}, \methodname{} constructs biased and semantic auxiliary prompt sets ($A_B$ and $A_S$) and performs SVD based bidirectional disentanglement. Given ``a bakery without bread,'' the method selectively removes only bread-related semantics entangled within the bakery representation. Through semantic recomposition and residual style restoration before denoising, \methodname{} suppresses undesired semantics while preserving contextual structure and visual style. This representation-level intervention is supported by theoretical analysis, and fundamentally differs from methods relying on attention manipulation~\cite{chen2026normalized,biswas2025cure} or iterative sampling optimization~\cite{na2025diffusion}.

% [6] 일반화 및 적용성
Because the proposed method requires only a one-time auxiliary prompt construction stage and operates on low-dimensional embedding subspaces, \methodname{} introduces negligible additional computational overhead during inference. 
Furthermore, representation-level semantic decomposition naturally generalizes from object-level negation to more abstract concept-level negation scenarios. Without requiring parameter updates or additional fine-tuning, \methodname{} is model-agnostic and can be seamlessly integrated into diverse diffusion architectures, as demonstrated in both U-Net~\cite{rombach2022high} and DiT based models~\cite{flux2024,podell2024sdxl}.

% [7] Contributions
\noindent The main contributions of this paper are as follows:
\begin{itemize}[leftmargin=1.2em, itemsep=2pt, topsep=4pt]
\item A simple yet effective training-free negation-aware generation method is proposed from the perspective of representation-level semantic disentanglement, which is model-agnostic and applicable to diverse diffusion architectures, including both U-Net and DiT based models.
\item An auxiliary prompt-driven bidirectional SVD strategy is introduced to disentangle bias-related semantics from text representations, while preserving visual style through semantic recomposition and residual restoration, supported by formal theoretical analysis.
\item Extensive experiments on object- and concept-level negation benchmarks demonstrate consistent improvements over existing baselines, achieving up to 27.4\% improvement in TIFA score and 2.7\% gain in PickScore, while consistently improving negation understanding in both conventional U-Net and recent DiT based diffusion architectures.
\end{itemize}

\section{Related Work} \label{section_2}
\paragraph{Text-Image Alignment.}
T2I research has primarily focused on improving text-image alignment through techniques such as reward optimization~\cite{na2025diffusion}, representation editing~\cite{arad2024refact}, and attention strength modulation~\cite{park2026pedit}.
While recent approaches such as DATE~\cite{na2025diffusion} and SoftREPA~\cite{lee2025aligning} further improve alignment, they primarily target general text-image correspondence rather than the representation-level semantic entanglement inherent in negation prompts.
More recently, negation-specific methods have emerged. SupEOT~\cite{li2023get} performs explicit object suppression through U-Net-based interventions, while NAG~\cite{chen2026normalized} employs attention extrapolation in transformer-based architectures. However, these methods rely on architecture-specific mechanisms, limiting their applicability in heterogeneous diffusion backbones.

\paragraph{Semantic Erasure.}
Concept erasure in diffusion models typically involves model editing~\cite{lu2024mace,gandikota2023erasing,biswas2025cure} or text editing~\cite{li2023get,Han-2025-CCRT}. While model editing often irreversibly alters model parameters and may affect generalization, text editing primarily performs global concept suppression but struggles with conditional negation, which requires context-aware selective removal. Although recent methods attempt to mitigate unintended side effects by localizing the erasure, such as leveraging graph-based spaces~\cite{bui2025fantastic} or null-space editing~\cite{li2025speed}, they still fundamentally rely on direct suppression of semantic directions.
%%%%%%%%%%%%%%%%%%%%%%%%%%%%%%%%%%%%%%%%%%%%%%%%%%%%%%%%%%%%%%%%%%%%%%%%%%%%%%%%%%%%%%%%%%%%%%%%%%%%%%%%%%%%%%%%%%%

\section{Proposed Method} \label{section_3}
\subsection{Motivation}
\label{section_3_1}
Recent diffusion models~\cite{podell2024sdxl,flux2024} have substantially improved text understanding and text-image alignment capabilities. Furthermore, recent architectures increasingly leverage multiple text encoders, such as CLIP~\cite{radford2021learning} and T5~\cite{raffel2020exploring}, to jointly capture global semantic context and fine-grained textual details. Despite these improvements, negation prompts remain challenging because removing a target concept requires suppressing not only explicit textual information but also implicit semantic priors learned from large-scale training data.

Fig.~\ref{figure_2}(a) illustrates this phenomenon using the prompt ``a bakery with no bread.'' Although the prompt explicitly specifies the absence of bread, the generated image still contains bread-related content. This failure suggests that the concept of a bakery is strongly associated with bread through implicit data bias, causing bias-aligned semantics to dominate over the intended negation constraint\cite{kang2025fairgen,orgad2023editing}. In contrast, as shown in Fig.~\ref{figure_2}(b), prompts with weaker semantic associations naturally preserve the desired semantics while successfully removing the target concept.

A straightforward solution is to project text embeddings into a semantic-preserving subspace using SVD based disentanglement. However, as shown in Fig.~\ref{figure_2}(c), uni-directional projection often removes not only bias-related representations but also contextual semantics and visual style information. While the undesirable concept is partially suppressed, important characteristics such as scene identity and visual appearance can also be lost.

To address this limitation, we propose a bidirectional SVD formulation that explicitly separates bias-aligned semantics from semantic-preserving representations while additionally restoring residual stylistic information. As shown in Fig.~\ref{figure_2}(d), the proposed formulation removes implicit bias while preserving both semantic consistency and visual style, leading to faithful negation-aware generation.

\begin{figure}[t]
\centering
\includegraphics[width=\linewidth]{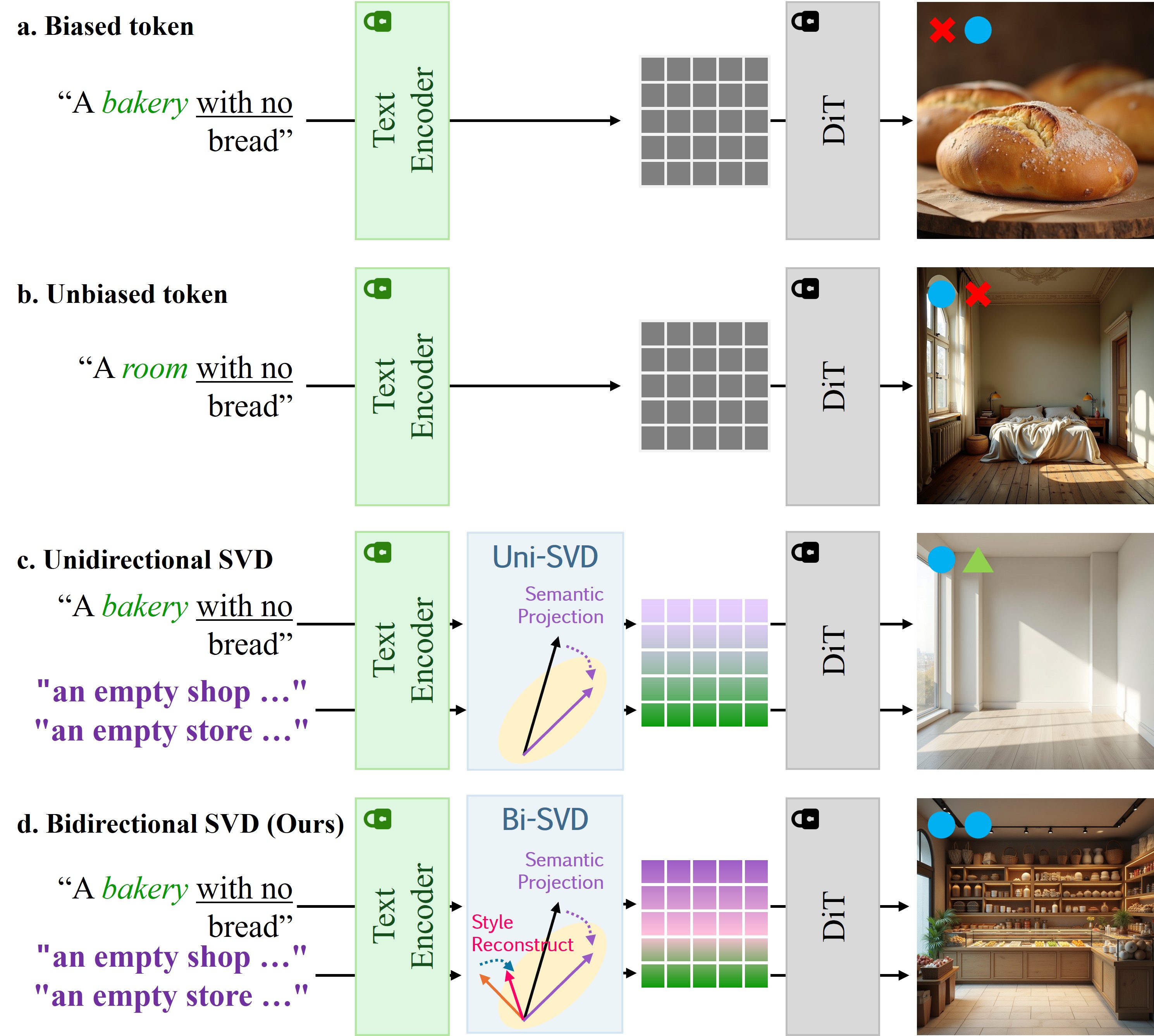}
\caption{Comparison of semantic disentanglement in negation.
(a) Strongly coupled: negation fails,
(b) Weakly coupled: style lost,
(c) Uni-SVD: style degraded, and
(d) Bi-SVD (Ours): negation succeeds, style preserved.
For each result, the left mark indicates whether the negation constraint is satisfied, and the right mark indicates whether the original contextual style (e.g., bakery interior) is preserved.}
\label{figure_2}
\end{figure}

\begin{figure*}[t]
\centering
\includegraphics[width=\textwidth]{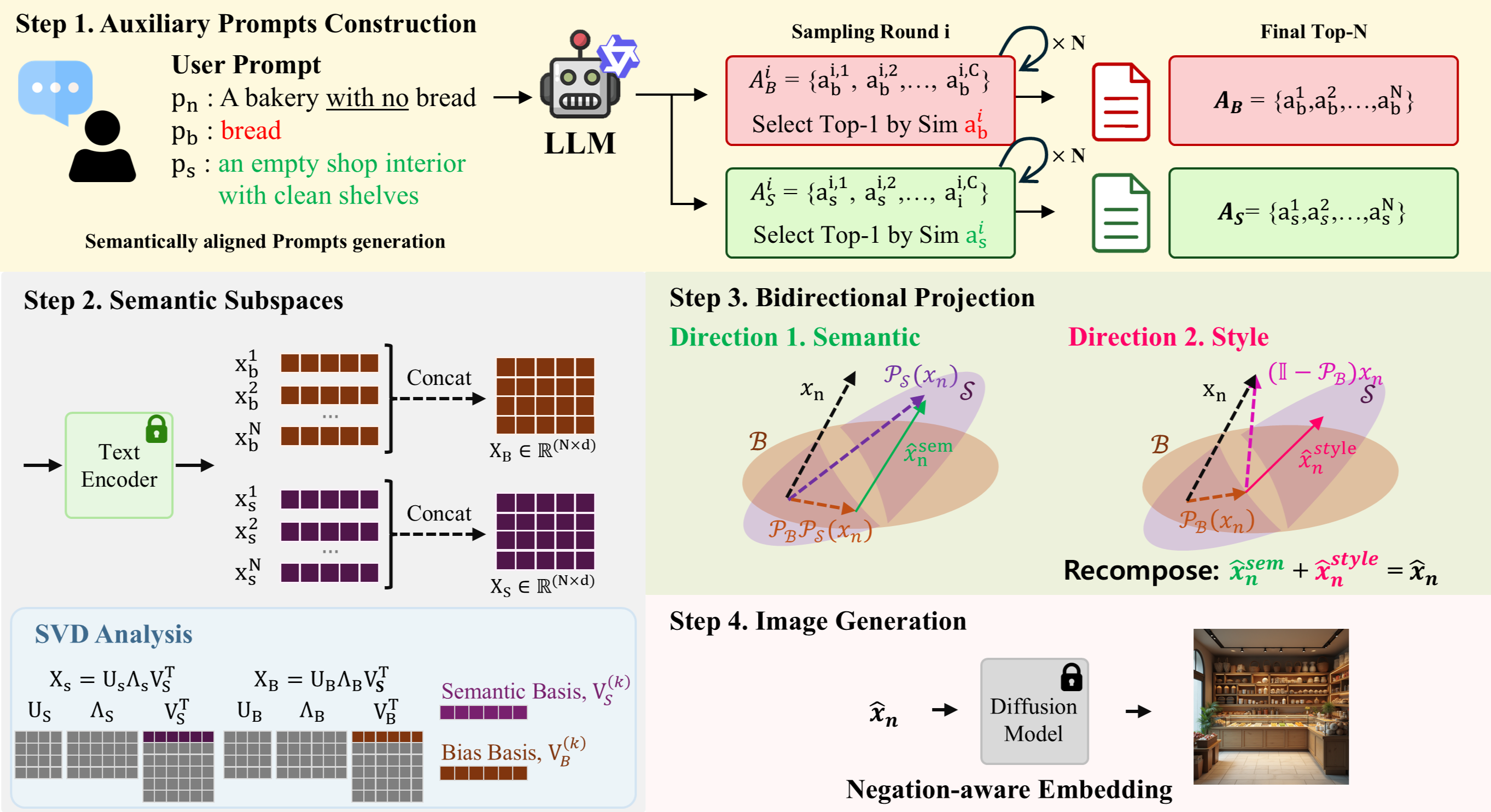}
\caption{Overview of \methodname{}. (Step 1) Biased and semantic auxiliary prompts are iteratively generated via an LLM and ranked by semantic similarity. (Step 2) Text embeddings of auxiliary prompts are encoded and decomposed via SVD to estimate the semantic basis $V^{(k)}_\mathcal{S}$ and bias basis $V^{(k)}_\mathcal{B}$. (Step 3) Bidirectional projection constructs $\hat{x}_n = \hat{x}_n^{\mathrm{sem}} + \hat{x}_n^{\mathrm{style}}$ via two directions. Direction 1 removes bias after semantic alignment, and Direction 2 restores style after bias removal. (Step 4) The recomposed negation-aware embedding $\hat{x}_n$ is applied for image generation.}
\label{figure_3}
\end{figure*}

\subsection{Negation-aware Text Embedding} \label{section_3_2}
\paragraph{Problem Setup.} \label{section_3_2_1}
Fig.~\ref{figure_3} illustrates the overall pipeline. A key challenge in negation-aware generation is that negated semantics cannot be represented as an independent embedding, since target semantics are strongly entangled with contextual and stylistic semantics within pretrained text representations. For example, in ``bakery with no bread,'' bread-related semantics remain deeply coupled with bakery-related contextual semantics due to pretrained co-occurrence priors (more details are in Appendix~\ref{section_B}). Consequently, direct suppression often removes useful contextual information or leaves residual target semantics in generated images.

% overview
To address this issue, given a negation prompt $p_n$, two auxiliary prompt sets are constructed to estimate disentangled semantic structures. The \textbf{semantic auxiliary prompt set $A_S$} captures information that should be preserved under negation constraints, whereas the \textbf{biased auxiliary prompt set $A_B$} models dominant semantics aligned with the negated target attribute. Each prompt set is transformed into an embedding matrix, from which a semantic-preserving subspace $\mathcal{S}$ and a bias subspace $\mathcal{B}$ are estimated via SVD. The original text embedding $x_n$, obtained from the negation prompt $p_n$, is then projected onto $\mathcal{S}$ while suppressing bias-aligned components within $\mathcal{B}$. Residual reconstruction further restores contextual and stylistic information, and the recomposed embedding $\hat{x}_n$ is injected into a frozen diffusion model for negation-aware image generation.

\begin{table}[h]
\centering
\small 
\setlength{\tabcolsep}{5pt}
\renewcommand{\arraystretch}{1.15}
\begin{tabularx}{\columnwidth}{@{}lX@{}}
\toprule
Field & Content \\
\midrule
Prompt & a shoe without laces \\
Biased prompt & shoelaces \\
Semantic prompt & a slip-on shoe without stripes \\
\midrule
\multicolumn{2}{@{}l@{}}{Biased auxiliary prompts} \\
\multicolumn{2}{@{}l@{}}{-- a sneaker with clearly tied shoelaces} \\
\multicolumn{2}{@{}l@{}}{-- a lace-up shoe with threaded eyelets} \\
\midrule
\multicolumn{2}{@{}l@{}}{Semantic auxiliary prompts} \\
\multicolumn{2}{@{}l@{}}{-- a slip-on shoe with a clean seamless upper} \\
\multicolumn{2}{@{}l@{}}{-- a minimalist loafer with no fastening elements} \\
\bottomrule
\end{tabularx}
\caption{Examples of auxiliary prompt.}
\label{table_1} \end{table}

\paragraph{Auxiliary Prompt Construction.} \label{section_3_2_2}

In object-level negation, it is difficult to explicitly define a semantic state where the target attribute is completely removed. For example, in ''a bakery with no bread'',  bread-related semantics remain strongly entangled with the bakery representation, making direct construction of a bread-free embedding highly challenging (See \S~\ref{section_B_1}). Moreover, semantically similar directions may still preserve residual bias even after partial suppression~\cite{yin2019semantics,huang2025implicit}. Therefore, instead of relying solely on a single negation prompt, we construct auxiliary prompt sets to estimate stable bias-reduced semantic structures.

Specifically, we construct a biased auxiliary prompt set $A_B = \{a_b^i\}_{i=1}^{N}$ and a semantic auxiliary prompt set $A_S = \{a_s^i\}_{i=1}^{N}$ to refine the representation of the negation prompt $p_n$. 
To guide this process, we define the biased prompt $p_b$ and the semantic prompt $p_s$. While $p_b$ captures the target attribute to be removed (e.g., bread) along with its deeply entangled bias components, $p_s$ preserves the desired contextual meaning without the target (e.g., an empty bakery interior). 
To ensure diversity and stability, the auxiliary prompts are iteratively generated using an LLM (e.g., Qwen2.5-Instruct) conditioned on $p_b$ and $p_s$, respectively. In each of the $N$ rounds, the LLM generates $C$ candidates, and the one maximizing semantic similarity to its corresponding target concept is selected. Detailed LLM instructions and specific examples of these prompts can be found in Appendix~\ref{section_D_1} and Table~\ref{table_1}.

\paragraph{Semantic Subspaces.} \label{section_3_2_3}

% Semantic Subspaces
% S: Semantic-Preserving Subspace
% B: Bias Subspace

% We first include the original negation prompt $p_n$ in both sets to preserve the target editing context. 
Each prompt in $A_B$ and $A_S$, along with the original negation prompt $p_n$, is transformed into the embedding space using a frozen text encoder, yielding the biased embedding matrix $X_B \in \mathbb{R}^{N \times d}$ and the semantic embedding matrix $X_S \in \mathbb{R}^{N\times d}$, where $d$ denotes the embedding dimension. To extract dominant semantic directions, we apply SVD:
\begin{equation}
X_S = U_S \Sigma_S V_S^\top,
\qquad
X_B = U_B \Sigma_B V_B^\top,
\label{equation_1}
\end{equation}
where the column vectors of $V_S$ and $V_B$ represent orthogonal bases of the semantic-preserving subspace and bias subspace, respectively, while $\Sigma_S$ and $\Sigma_B$ contain singular values indicating the importance of each semantic direction. Here, $k$ denotes the number of dominant semantic components used to construct the subspace. In practice, we empirically found that using only the first singular component ($k=1$), corresponding to the largest singular value, consistently produced the best performance across different settings. Consequently, the semantic-preserving subspace $\mathcal{S}$ and bias subspace $\mathcal{B}$ are defined by $V_S^{(k)}$ and $V_B^{(k)}$, respectively. Here, $\mathcal{S}$ represents the semantic direction that preserves the desired semantic structure after removing the target attribute, whereas $\mathcal{B}$ captures the implicit semantic bias direction strongly entangled with the target attribute.

\begin{figure*}[t]
\centering
\includegraphics[width=\textwidth]{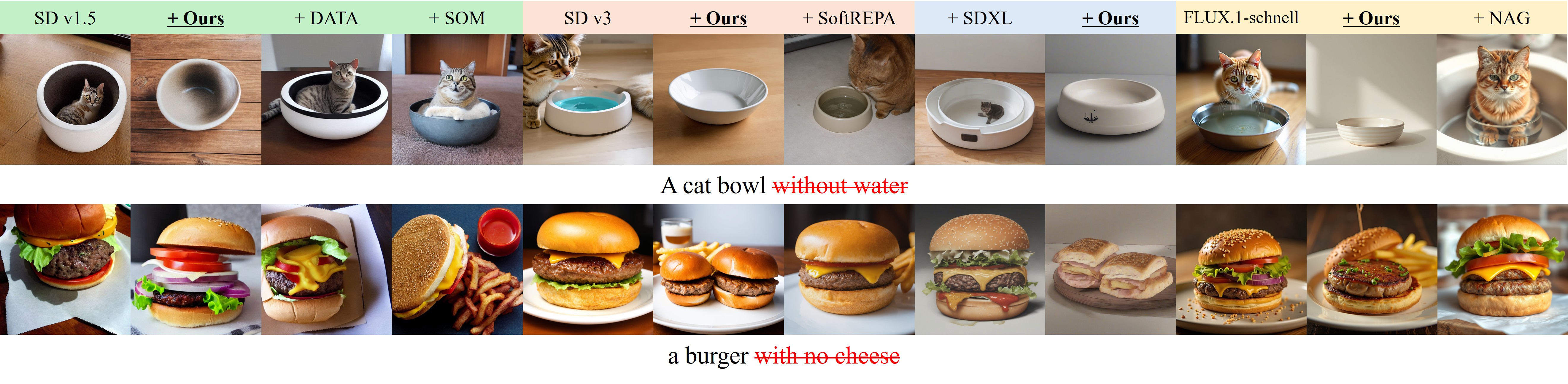}
\caption{Comparison of object negation in different diffusion backbones. Existing methods often fail to completely remove the negated concept or degrade visual style, whereas the proposed method achieves faithful negation while preserving semantic structure and visual fidelity. Red strike-through text indicates the negated concepts.}
\label{figure_4}
\end{figure*}

\paragraph{Bidirectional Projection.} \label{section_3_2_4}

Decomposing text embeddings requires careful structural manipulation, as implicit biases and contextual features are tightly coupled. Consequently, a naive one-directional projection inevitably creates a significant trade-off between bias mitigation and contextual fidelity. To overcome this limitation, we propose a bidirectional projection mechanism that simultaneously enhances semantic alignment while suppressing residual bias components.

Specifically, the original negation embedding $x_n$ is first projected onto the semantic-preserving subspace $\mathcal{S}$ to extract the desired semantic representation, and the residual bias component is subsequently removed to obtain the bias-reduced semantic representation:
\begin{equation}
\label{equation_2}
\hat{x}_n^{\mathrm{sem}} = (\mathbb{I} - \mathcal{P}_\mathcal{B})\mathcal{P}_\mathcal{S}(x_n),
\end{equation}
where $\mathbb{I}$ denotes the identity operator, and $\mathcal{P}_\mathcal{S}$ and $\mathcal{P}_\mathcal{B}$ represent the orthogonal projection operators onto $\mathcal{S}$ and $\mathcal{B}$  (see Appendix~\ref{section_D_2} for details).

However, excessive semantic alignment may inadvertently remove stylistic information. To compensate for this effect, we utilize the residual representation after removing the bias component from the original embedding and project it onto the semantic-preserving subspace to restore contextual and style-related information:
\begin{equation}
\label{equation_3}
\hat{x}_n^{\mathrm{style}} = \mathcal{P}_\mathcal{S}(\mathbb{I} - \mathcal{P}_\mathcal{B})(x_n).
\end{equation}

Finally, the semantic-aligned representation and the restored residual representation are combined to construct the final negation-aware embedding:
\begin{equation}
\hat{x}_n = \hat{x}_n^{\mathrm{sem}} + \hat{x}_n^{\mathrm{style}}.
\label{equation_4}
\end{equation}
The recomposed negation-aware embedding $\hat{x}_n$ is then injected into the frozen diffusion model, enabling negation-aware image generation while preserving the original semantic structure and visual style.

\subsection{Theoretical Analysis} \label{section_3_3}
In this section, we theoretically analyze the underlying mechanisms of Uni-directional SVD (Uni-SVD) and Bidirectional SVD (Bi-SVD), confirming how \methodname{} resolves the fundamental trade-off between bias suppression and style preservation. The detailed formal statement and proofs are provided in Appendix~\ref{section_C}. We assume that a text embedding $x_n$ is decomposed into entangled semantic, bias, and stylistic components within a shared embedding space: $x_n = x_n^{\mathrm{sem}} + x_n^{\mathrm{bias}} + x_n^{\mathrm{style}}$.

\begin{theorem}[Uni-SVD: Style Degradation] \label{theorem_1}
Uni-SVD mitigates bias semantics through the operator 
$\mathcal{T}_{\mathrm{uni}} = \mathbb{I} - \mathcal{P}_\mathcal{B}$. 
However, because stylistic and biased components are inherently 
entangled (i.e., the stylistic representation shares significant 
overlap with the bias subspace $\mathcal{B}$), $\mathcal{T}_{\mathrm{uni}}$ 
inadvertently diminishes these stylistic features alongside the target 
bias, yielding $\hat{x}_n^{\mathrm{style}} = 0$. Consequently, applying 
a one-directional projection often leads to a noticeable degradation 
in visual style when bias suppression is enforced.
\end{theorem}

\begin{theorem}[Bi-SVD: Style Recovery] \label{theorem_2}
In contrast, the bidirectional operator 
$\mathcal{T}_{\mathrm{bi}}= (\mathbb{I} - \mathcal{P}_\mathcal{B})\mathcal{P}_\mathcal{S}
+\mathcal{P}_\mathcal{S}(\mathbb{I} - \mathcal{P}_\mathcal{B})$ 
overcomes this limitation through cross-subspace projection. 
Specifically, by additionally constructing $\hat{x}_n^{\mathrm{style}}$ 
via cross-subspace projection, Bi-SVD explicitly rescues the stylistic 
content that Uni-SVD destroys.
\end{theorem}

\begin{table*}[t]
\centering
\begingroup
\setlength{\tabcolsep}{4pt}
\renewcommand{\arraystretch}{0.8}
\begin{adjustbox}{max width=\textwidth}
\setlength{\tabcolsep}{2.5pt}
\begin{tabular}{lcccccc r@{~}l r@{~}l r@{~}l}
\toprule
\multirow{2}{*}{\textbf{Model}}
& \multicolumn{8}{c}{TIFA Score $\uparrow$}
& \multicolumn{4}{c}{Image Quality} \\ 
& missing & empty & nothing & nobody & without & no
& \multicolumn{2}{c}{mean}
& \multicolumn{2}{c}{PickScore $\uparrow$}
& \multicolumn{2}{c}{FID $\downarrow$} \\
\midrule
SD v1.5 & 0.142 & 0.252 & 0.525 & 0.950 & 0.166 & 0.722 & 0.460 & & 21.245 & & 195.207 & \\
SD v1.5 + DATE & 0.145 & 0.255 & 0.525 & 0.818 & 0.166 & 0.722 & 0.439 & {\scriptsize\textcolor{softred}{(-0.021)}} & 21.229 & {\scriptsize\textcolor{softred}{(-0.016)}} & 194.615 & {\scriptsize\textcolor{softblue}{(-0.592)}} \\
SD v1.5 + SupEOT & 0.155 & 0.250 & 0.525 & 0.955 & 0.250 & 0.770 & 0.484 & {\scriptsize\textcolor{softblue}{(+0.024)}} & 21.174 & {\scriptsize\textcolor{softred}{(-0.071)}} & 208.241 & {\scriptsize\textcolor{softred}{(+13.034)}} \\
\rowcolor{softpurple}
SD v1.5 + Ours & 0.358 & 0.455 & 0.526 & 0.962 & 0.458 & 0.812 & \textbf{0.595} & {\scriptsize\textcolor{softblue}{(+0.135)}} & 21.815 & {\scriptsize\textcolor{softblue}{(+0.570)}} & 193.091 & {\scriptsize\textcolor{softblue}{(-2.116)}} \\
\midrule
SD v3 & 0.151 & 0.250 & 0.752 & 0.909 & 0.250 & 0.762 & 0.512 & & 21.751 & & 201.016 & \\
SD v3 + SoftREPA & 0.157 & 0.281 & 0.763 & 0.909 & 0.167 & 0.738 & 0.503 & {\scriptsize\textcolor{softred}{(-0.009)}} & 21.751 & {\scriptsize\textcolor{softblue}{(+0.000)}} & 199.158 & {\scriptsize\textcolor{softblue}{(-1.858)}} \\
\rowcolor{softpurple}
SD v3 + Ours
& 0.362 & 0.486 & 0.774 & 0.916 & 0.372 & 0.748 & \textbf{0.610} & {\scriptsize\textcolor{softblue}{(+0.098)}} & 22.134 & {\scriptsize\textcolor{softblue}{(+0.383)}}
& 197.582 & {\scriptsize\textcolor{softblue}{(-3.434)}} \\
\midrule
SDXL & 0.412 & 0.438 & 0.421 & 0.747 & 0.405 & 0.432 & 0.476 & & 20.806 & & 202.049 & \\
\rowcolor{softpurple}
SDXL + Ours & 0.461 & 0.491 & 0.452 & 0.782 & 0.447 & 0.475 & \textbf{0.518} & {\scriptsize\textcolor{softblue}{(+0.042)}} & 20.890 & {\scriptsize\textcolor{softblue}{(+0.084)}} & 201.934 & {\scriptsize\textcolor{softblue}{(-0.115)}} \\
\midrule
FLUX.1       & 0.845 & 0.450 & 0.855 & 0.805 & 0.250 & 0.762 & 0.661 & & 22.265 & & 194.709 & \\
FLUX.1 + NAG & 0.852 & 0.820 & 0.855 & 0.820 & 0.667 & 0.942 & 0.826 & {\scriptsize\textcolor{softblue}{(+0.165)}} & 22.018 & {\scriptsize\textcolor{softred}{(-0.247)}} & 194.965 & {\scriptsize\textcolor{softred}{(+0.256)}} \\
\rowcolor{softpurple}
FLUX.1 + Ours & 0.765 & 0.951 & 0.868 & 1.000 & 0.667 & 0.802 & \textbf{0.842} & {\scriptsize\textcolor{softblue}{(+0.181)}} & 22.268 & {\scriptsize\textcolor{softblue}{(+0.003)}} & 192.375 & {\scriptsize\textcolor{softblue}{(-2.334)}} \\
\bottomrule
\end{tabular}
\end{adjustbox}
\caption{
Quantitative comparison of object negation understanding and image quality in different generative models.
\textcolor{softblue}{Blue} and \textcolor{softred}{red} values indicate improvements and degradations relative to the corresponding base model, respectively, while \textbf{bold} values indicate the best performance in each comparison.
}
\label{table_2}
\endgroup
\end{table*}

\begin{table*}[t]
\centering
\begingroup
\setlength{\tabcolsep}{2.2pt}
\renewcommand{\arraystretch}{0.8}
\begin{adjustbox}{max width=\textwidth}
\begin{tabular}{l cc r@{\;}l r@{\;}l r@{\;}l cc r@{\;}l r@{\;}l r@{\;}l}
\toprule
\multirow{3}{*}{\textbf{Model}}
& \multicolumn{8}{c}{Concept ``Picasso'' Negation}
& \multicolumn{8}{c}{Concept ``Van Gogh'' Negation} \\
\cmidrule(lr){2-9} \cmidrule(lr){10-17}
& LPIPS$\uparrow$ & TIFA$\uparrow$
& \multicolumn{2}{c}{CLIP w/ C$\downarrow$}
& \multicolumn{2}{c}{CLIP w/o C$\uparrow$}
& \multicolumn{2}{c}{PickScore$\uparrow$}
& LPIPS$\uparrow$ & TIFA$\uparrow$
& \multicolumn{2}{c}{CLIP w/ C$\downarrow$}
& \multicolumn{2}{c}{CLIP w/o C$\uparrow$}
& \multicolumn{2}{c}{PickScore$\uparrow$} \\
\midrule

SD v1.4
& -- & 0.083
& 0.507 &
& 0.521 &
& 18.861 &
& -- & 0.125
& 0.743 &
& 0.843 &
& 20.330 & \\

SD v1.4 + SPEED
& \textbf{0.199} & 0.036
& 0.495 & {\scriptsize\textcolor{softblue}{(-0.012)}}
& 0.515 & {\scriptsize\textcolor{softred}{(-0.006)}}
& 19.479 & {\scriptsize\textcolor{softblue}{(+0.618)}}
& 0.134 & 0.042
& 0.584 & {\scriptsize\textcolor{softblue}{(-0.159)}}
& 0.684 & {\scriptsize\textcolor{softred}{(-0.159)}}
& 20.330 & {\scriptsize\textcolor{softblue}{(+0.000)}} \\

SD v1.4 + MACE
& 0.181 & 0.175
& 0.463 & {\scriptsize\textcolor{softblue}{(-0.044)}}
& 0.563 & {\scriptsize\textcolor{softblue}{(+0.042)}}
& 19.841 & {\scriptsize\textcolor{softblue}{(+0.980)}}
& 0.139 & 0.158
& 0.654 & {\scriptsize\textcolor{softblue}{(-0.089)}}
& 0.754 & {\scriptsize\textcolor{softred}{(-0.089)}}
& 20.870 & {\scriptsize\textcolor{softblue}{(+0.540)}} \\

\rowcolor{softpurple}
SD v1.4 + Ours
& 0.181 & \textbf{0.179}
& \textbf{0.341} & {\scriptsize\textcolor{softblue}{(-0.165)}}
& \textbf{0.612} & {\scriptsize\textcolor{softblue}{(+0.091)}}
& \textbf{20.330} & {\scriptsize\textcolor{softblue}{(+1.469)}}
& \textbf{0.146} & \textbf{0.167}
& \textbf{0.339} & {\scriptsize\textcolor{softblue}{(-0.404)}}
& \textbf{0.854} & {\scriptsize\textcolor{softblue}{(+0.011)}}
& \textbf{21.420} & {\scriptsize\textcolor{softblue}{(+1.090)}} \\

\midrule

SD v1.5
& -- & 0.112
& 0.772 &
& 0.772 &
& 19.651 &
& -- & 0.063
& 0.749 &
& 0.749 &
& 21.708 & \\

SD v1.5 + RECE
& 0.202 & 0.139
& 0.598 & {\scriptsize\textcolor{softblue}{(-0.174)}}
& 0.598 & {\scriptsize\textcolor{softred}{(-0.174)}}
& 19.165 & {\scriptsize\textcolor{softred}{(-0.486)}}
& 0.143 & 0.103
& 0.465 & {\scriptsize\textcolor{softblue}{(-0.284)}}
& 0.465 & {\scriptsize\textcolor{softred}{(-0.284)}}
& 18.549 & {\scriptsize\textcolor{softred}{(-3.159)}} \\

\rowcolor{softpurple}
SD v1.5 + Ours
& \textbf{0.233} & \textbf{0.143}
& \textbf{0.346} & {\scriptsize\textcolor{softblue}{(-0.426)}}
& \textbf{0.834} & {\scriptsize\textcolor{softblue}{(+0.156)}}
& \textbf{20.367} & {\scriptsize\textcolor{softblue}{(+0.716)}}
& \textbf{0.148} & \textbf{0.155} 
& \textbf{0.344} & {\scriptsize\textcolor{softblue}{(-0.405)}}
& \textbf{0.864} & {\scriptsize\textcolor{softblue}{(+0.124)}}
& \textbf{21.736} & {\scriptsize\textcolor{softblue}{(+0.028)}} \\

\midrule

SDXL
& -- & 0.428 & 0.594 & 
& 0.415 &
& 20.298 &
& -- & 0.489
& 0.607 & & 0.607 &
& 19.428 & \\

\rowcolor{softpurple}
SDXL + Ours & \textbf{0.216} & \textbf{0.462} & \textbf{0.515} & {\scriptsize\textcolor{softblue}{(-0.079)}}
& \textbf{0.494} & {\scriptsize\textcolor{softblue}{(+0.079)}}
& \textbf{21.264} & {\scriptsize\textcolor{softblue}{(+0.966)}}
& \textbf{0.201} & \textbf{0.527}
& \textbf{0.721} & {\scriptsize\textcolor{softred}{(+0.114)}}
& \textbf{0.721} & {\scriptsize\textcolor{softblue}{(+0.114)}}
& \textbf{20.471} & {\scriptsize\textcolor{softblue}{(+1.043)}} \\

\midrule

FLUX.1
& -- & 0.444
& 0.473 &
& 0.553 &
& 22.075 &
& -- & 0.205
& 0.732 &
& 0.732 &
& 22.109 & \\

FLUX.1 + Erasing
& 0.127 & 0.517
& 0.464 & {\scriptsize\textcolor{softblue}{(-0.009)}}
& 0.511 & {\scriptsize\textcolor{softred}{(-0.042)}}
& 21.984 & {\scriptsize\textcolor{softred}{(-0.091)}}
& 0.169 & 0.112
& 0.735 & {\scriptsize\textcolor{softred}{(+0.003)}}
& 0.735 & {\scriptsize\textcolor{softblue}{(+0.003)}}
& 22.198 & {\scriptsize\textcolor{softblue}{(+0.089)}} \\

FLUX.1 + UCE
& 0.087 & 0.464
& 0.453 & {\scriptsize\textcolor{softblue}{(-0.020)}}
& 0.564 & {\scriptsize\textcolor{softblue}{(+0.011)}}
& 22.189 & {\scriptsize\textcolor{softblue}{(+0.114)}}
& 0.134 & 0.422
& 0.742 & {\scriptsize\textcolor{softred}{(+0.010)}}
& 0.742 & {\scriptsize\textcolor{softblue}{(+0.010)}}
& 22.248 & {\scriptsize\textcolor{softblue}{(+0.139)}} \\

\rowcolor{softpurple}
FLUX.1 + Ours
& \textbf{0.175} & \textbf{0.525}
& \textbf{0.411} & {\scriptsize\textcolor{softblue}{(-0.062)}}
& \textbf{0.573} & {\scriptsize\textcolor{softblue}{(+0.020)}}
& \textbf{22.499} & {\scriptsize\textcolor{softblue}{(+0.424)}}
& \textbf{0.173} & \textbf{0.459}
& \textbf{0.581} & {\scriptsize\textcolor{softblue}{(-0.151)}}
& \textbf{0.781} & {\scriptsize\textcolor{softblue}{(+0.049)}}
& \textbf{23.109} & {\scriptsize\textcolor{softblue}{(+1.000)}} \\

\bottomrule
\end{tabular}
\end{adjustbox}
\caption{Quantitative comparison of concept negation understanding and image quality in different generative models.}
\label{table_3}
\endgroup
\end{table*}

\section{Experiments} \label{section_5}
\begin{figure*}[t]
\centering
\includegraphics[width=\textwidth]{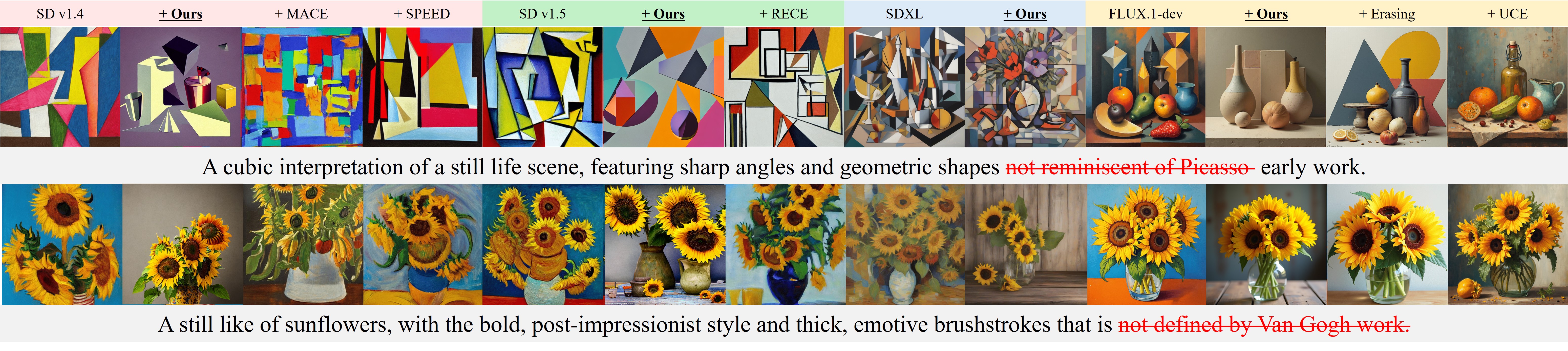}
\caption{Qualitative comparisons of concept-level negation in different diffusion backbones.}
\label{figure_5}
\end{figure*}

\begin{figure}[t]
\centering
\includegraphics[width=\linewidth]{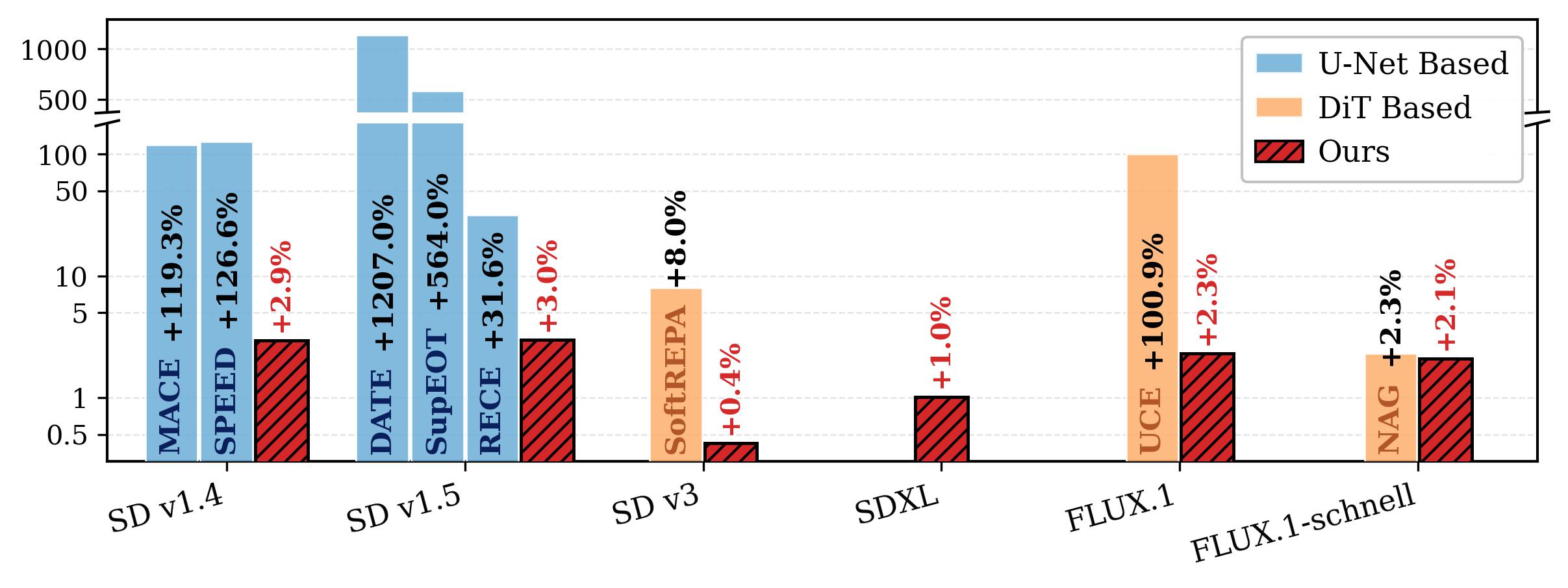}
\caption{Inference-time overhead comparison in a log scale.}
\label{figure_6}
\end{figure}

\subsection{Object Negation} \label{section_4_1}
\methodname{} is applied to diffusion backbones including the U-Net based Stable Diffusion v1.5 (SD v1.5)~\cite{rombach2022high} and Transformer based Rectified Flow models such as Stable Diffusion v3 (SD v3)~\cite{esser2024scaling}, SDXL~\cite{podell2024sdxl}, and FLUX.1~\cite{flux2024}.
Performance is evaluated using TIFA~\cite{hu2023tifa}, a VQA based metric for assessing text-image semantic alignment, together with FID and PickScore~\cite{kirstain2023pick} to evaluate visual quality and the effectiveness of style preservation.

We construct an evaluation set of 480 prompts starting from sentences in the ImageNet validation set.
Using Qwen2.5-Instruct, we transform the original sentences into object-negation prompts containing explicit negation expressions such as 'without' and 'no.'
These prompts are used to evaluate whether each method can selectively exclude the negated object while preserving the remaining semantics of the prompt.

Performance comparisons are conducted against representative baseline methods in different diffusion backbones, including the Stable Diffusion v1.5 base model and its variants DATE~\cite{na2025diffusion} and SupEOT~\cite{li2023get}, the SD v3 base model and its variant SoftREPA~\cite{lee2025aligning}, as well as SDXL and FLUX.1 with its variant NAG~\cite{chen2026normalized}. The quantitative results are summarized in Table~\ref{table_2}. \textcolor{softblue}{Blue} values indicate improvements relative to the corresponding base model, while \textcolor{softred}{Red} values represent performance degradation.

Table~\ref{table_2} presents quantitative comparisons of object negation understanding and image quality. Existing suppression based methods such as SupEOT and NAG improve negation understanding over baseline models, but often struggle to maintain image quality consistently. In contrast, the proposed method effectively improves both negation alignment and generation quality across different diffusion backbones. Under the SD v1.5 setting, \methodname{} achieves the highest mean TIFA score of 0.595 while improving PickScore from 21.245 to 21.815 and reducing FID from 195.207 to 193.091. Similar trends are consistently observed in SD v3, SDXL, and FLUX.1, demonstrating that auxiliary prompts effectively enhance negation understanding while preserving visual quality.

As shown in Fig.~\ref{figure_4}, existing methods tend to follow explicit sentence expressions. In contrast, negative attributes are effectively removed, producing semantically cleaner results, such as an empty cat bowl without residual water and a burger without cheese.

\subsection{Concept Negation}
\label{section_4_2}
For concept negation, we compare training-based methods, including MACE~\cite{lu2024mace} built on Stable Diffusion v1.4 (SD v1.4), UCE~\cite{gandikota2024unified}, and Erasing~\cite{gandikota2023erasing}, with training-free methods such as SPEED~\cite{li2025speed} and RECE~\cite{gong2024reliable}. Experiments are conducted on the artist negation benchmark~\cite{gandikota2024unified}; following CURE~\cite{biswas2025cure}, we select 10 representative artists with 20 prompts per artist, where concept negation targets artist-specific semantic characteristics rather than contextual object co-occurrences. Detailed experimental settings and additional results are provided in Appendix~\ref{section_F_2}.

\begin{table}[t]
\centering
\setlength{\tabcolsep}{4pt} % 열이 늘어났으므로 간격을 살짝 좁힙니다
\renewcommand{\arraystretch}{0.8}
\caption{Result of the user study.}
\label{table_4}
\begin{adjustbox}{width=\linewidth}
\tiny
\setlength{\tabcolsep}{6pt} % 열 간격을 조절하여 조밀하게 만듭니다.
\begin{tabular}{lcccc}
\toprule
\multirow{2}{*}{Method} & \multicolumn{2}{c}{Object} & \multicolumn{2}{c}{Concept} \\
\cmidrule(lr){2-3} \cmidrule(lr){4-5}
& Align. & Vis. & Align. & Vis. \\
\midrule
Base    & 0.115 & 0.117 & 0.125 & 0.123 \\
Compare & 0.347 & 0.370 & 0.265 & 0.307 \\
\rowcolor{softpurple} Ours    & \textbf{0.536} & \textbf{0.512} & \textbf{0.609} & \textbf{0.569} \\
\bottomrule
\end{tabular}
\end{adjustbox}
\end{table}

% [4] 평가방법
% 4-1 % 4-2 % 4-3
We use LPIPS~\cite{zhang2018perceptual} to measure the perceptual difference between images generated by the original model and those generated after concept erasure, where a larger distance indicates stronger modification of the target concept and thus more effective concept removal. TIFA and PickScore follow the same protocol as object negation. For CLIP evaluation, we additionally measure similarity with and without concept terms (CLIP w/ C and CLIP w/o C), since concept terms may artificially increase similarity when the target concept remains in generated images.
% [5] 결과
Fig.~\ref{figure_5} and Table~\ref{table_3} show the results on concept negation. Existing methods often degrade semantic alignment and visual quality. In particular, SPEED and RECE significantly deteriorate both CLIP and PickScore. In contrast, the proposed method consistently preserves semantic alignment and image fidelity. Under the SDXL setting, CLIP improves from 0.415 to 0.494 for Picasso negation and from 0.607 to 0.721 for Van Gogh negation, with consistent PickScore improvements.

\paragraph{Computational Cost.}
Fig.~\ref{figure_6} compares inference time and performance. Since \methodname{} only performs lightweight SVD analysis on a small set of auxiliary prompts, it introduces negligible computational overhead while maintaining stable performance gains. Further analysis is in Appendix~\ref{section_H}.

\paragraph{User Study.}
As shown in Table~\ref{table_4}, \methodname{} achieves the highest human preference scores for both object- and concept-level negation (\textit{Compare} denotes NAG for object negation and Erasing for concept negation).
The evaluation involved 32 participants, with the sample size determined through a power analysis ($d=0.8$, power $=0.8$, $\alpha=0.05$).
Detailed evaluation protocols, score calculation, and statistical analyses are provided in Appendix~\ref{section_I}.

\begin{figure*}[t]
\centering
\includegraphics[width=\textwidth]{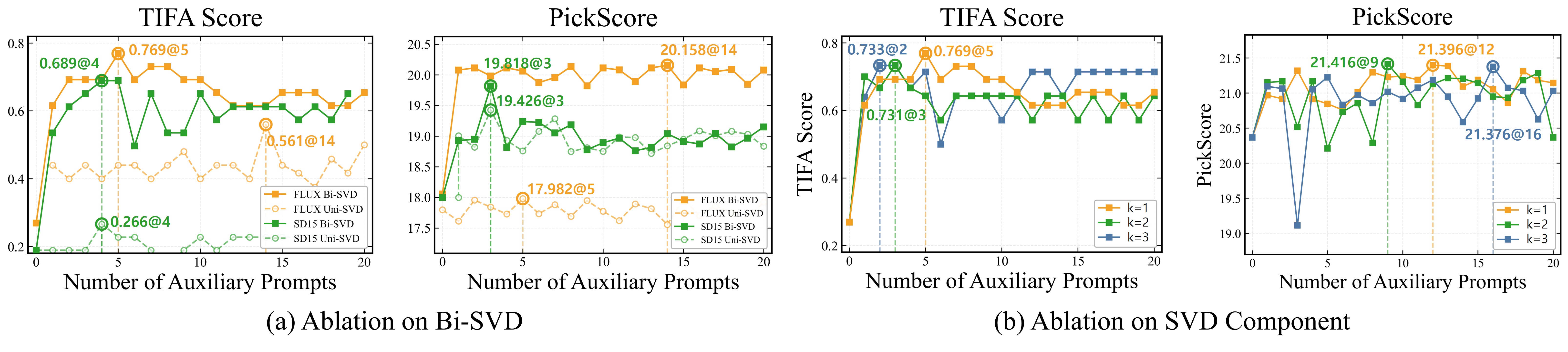}
\caption{Ablation studies on Bi-SVD and semantic basis extraction. (a) Comparison between Uni-SVD and Bi-SVD on FLUX.1 and SD v1.5. (b) Performance analysis with different numbers of SVD components on FLUX.1.}
\label{figure_7}
\end{figure*}

\subsection{Ablation Studies} \label{section_4_3}
% \paragraph{Effect of Bi-SVD}  \label{section_4_3_1}
% Fig.~\ref{figure_7}(a) compares Uni-SVD and Bi-SVD under varying numbers of auxiliary prompts. Both improve over the base model, indicating that semantic bias disentanglement is beneficial for negation-aware generation. However, Bi-SVD shows more stable semantic alignment and visual quality, reaching optimal performance with fewer auxiliary prompts, whereas Uni-SVD improves more slowly. This can be attributed to the bidirectional projection mechanism of Bi-SVD, which further removes residual biased semantic components after alignment to the semantic-preserving subspace (Eq.~\ref{equation_2}), while Uni-SVD relies on one-directional suppression that may insufficiently disentangle bias or remove useful semantic information. Similar trends are observed in PickScore. In SD v1.5, incomplete suppression in Uni-SVD occasionally preserved stylistic information, leading to relatively higher PickScore values, whereas excessive suppression in FLUX.1 sometimes degraded visual fidelity. In contrast, Bi-SVD maintained stable PickScore performance while improving semantic alignment. Additional quantitative results are provided in Appendix~\ref{section_F_1}.

\paragraph{Effect of Bi-SVD.} \label{section_4_3_1}
Fig.~\ref{figure_7}(a) compares Uni-SVD and Bi-SVD under varying numbers of auxiliary prompts. Both improve over the base model, indicating that semantic bias disentanglement is beneficial for negation-aware generation. However, Bi-SVD shows more stable semantic alignment and visual quality, reaching optimal performance with fewer auxiliary prompts, whereas Uni-SVD improves more slowly. This can be attributed to the bidirectional projection mechanism of Bi-SVD, which further removes residual biased semantic components after alignment to the semantic-preserving subspace (Eq.~\ref{equation_2}), while Uni-SVD relies on one-directional suppression that may insufficiently disentangle bias or remove useful semantic information. Similar trends are observed in PickScore. In SD v1.5, incomplete suppression in Uni-SVD occasionally preserves stylistic information, leading to relatively higher PickScore values, whereas excessive suppression in FLUX.1 sometimes degrades visual fidelity. In contrast, Bi-SVD maintains stable PickScore performance while improving semantic alignment.

To further examine the contribution of each branch, Table~\ref{table_5} compares Uni-SVD, the semantic branch only, the style branch only, and the full Bi-SVD. Neither individual branch achieves the performance of the full Bi-SVD across both backbones. In particular, Bi-SVD achieves the highest TIFA and PickScore on both FLUX.1 and SD v1.5, demonstrating that the two branches are complementary for semantic disentanglement and style preservation.

\begin{table}[t]
\centering
\caption{Ablation of individual branches in Bi-SVD in different diffusion backbones.}
\label{table_5}
\small
\begin{tabular}{llcc}
\toprule
Base &SVD Method& \multicolumn{2}{c}{\textbf{Score}} \\
\cmidrule(lr){3-4}
& & TIFA & PickScore \\
\midrule
\multirow{4}{*}{SD1.5}
& Uni-SVD              & 0.510 & 20.940 \\
& Semantic Branch      & 0.425 & 20.900 \\
& Style Branch         & 0.100 & 19.504 \\
& \textbf{Bi-SVD (Ours)} & \textbf{0.648} & \textbf{21.183} \\
\midrule
\multirow{4}{*}{FLUX}
& Uni-SVD              & 0.516 & 21.273 \\
& Semantic Branch      & 0.431 & 21.668 \\
& Style Branch         & 0.416 & 21.535 \\
& \textbf{Bi-SVD (Ours)} & \textbf{0.556} & \textbf{22.750} \\
\bottomrule
\end{tabular}
\end{table}

\paragraph{Effect of SVD Component.}\label{section_4_3_2}
Fig.~\ref{figure_7}(b) presents the performance differences under varying numbers of SVD components ($k$). Larger values of $k$ initially show higher TIFA scores, suggesting stronger negation behavior by capturing additional semantic variations associated with the target concept. However, as the number of auxiliary prompts increases, larger $k$ values exhibit bigger performance fluctuations and gradual degradation. This can be attributed to additional noisy components being introduced into the semantic basis together with useful semantic information. In contrast, $k=1$ consistently achieves more stable performance and the best overall results, as the first singular component mainly captures the dominant semantic direction while minimizing unnecessary noise. Similar trends are also observed in PickScore, where $k=1$ maintains more stable visual quality and stylistic consistency. Additional quantitative results are provided in Appendix~\ref{section_F_4}.

\paragraph{Extraction Method.}\label{section_4_3_3}
Fig.~\ref{figure_8} compares different extraction methods for semantic basis decomposition. Mean aggregation shows limited effectiveness due to insufficient disentanglement of biased semantics, while RPCA tends to produce unstable representations and degraded visual quality. In contrast, SVD more reliably estimates shared semantic structures across auxiliary prompts, leading to more consistent disentanglement behavior and better preservation of image quality. Detailed experimental results and discussion are provided in Appendix~\ref{section_F_3}.

\begin{figure}[t]
\centering
\includegraphics[width=\linewidth]{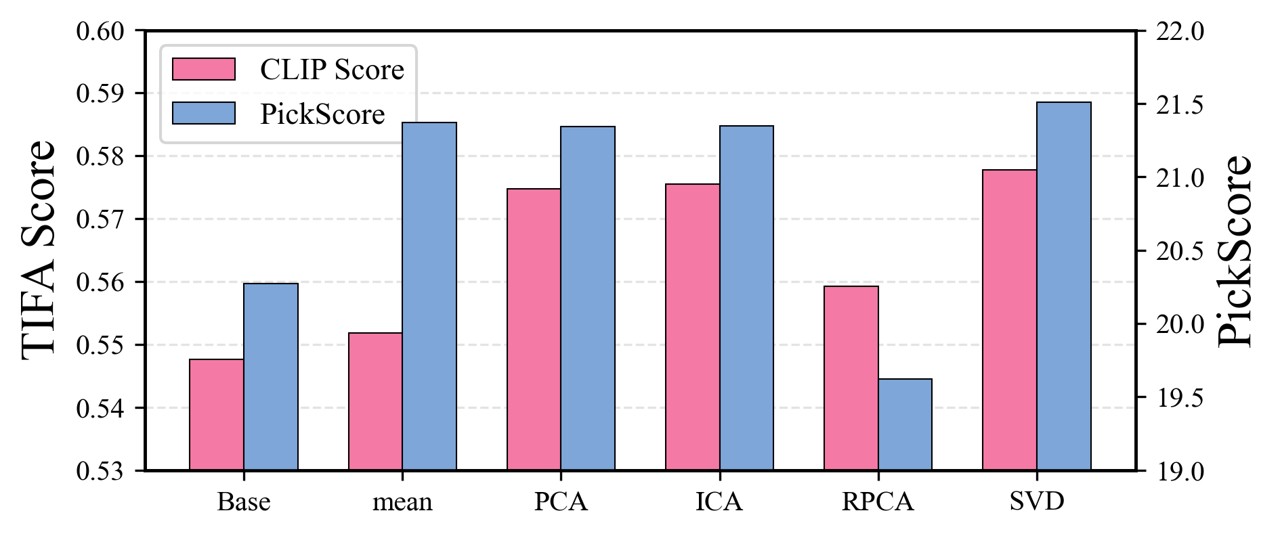}
\caption{Comparison of different extraction methods for semantic basis decomposition. SVD provides the most stable semantic disentanglement while preserving visual quality, whereas mean aggregation shows limited separation and RPCA exhibits unstable behavior.}
\label{figure_8}
\end{figure}

% % 1. LLM 도 동일한거 하나로 모으고
% % 2. Prompt 도 동일한거 하나로 모으자
% % 전체 width 는 살짝 줄이고 
% \begin{table}[t]
% \centering
% \caption{Robustness of auxiliary prompts across different LLMs, prompt types, and selection strategies.}
% \label{table_6}
% \small
% \begin{tabular}{lllcc}
% \toprule
% \textbf{LLM} & \textbf{Prompt} & \textbf{Selection} & \textbf{TIFA} & \textbf{PickScore} \\
% \midrule
% Qwen     & Noisy  & Random       & 0.526 & 21.626 \\
% Qwen     & Noisy  & Top-1        & 0.550 & 21.740 \\
% Qwen     & Manual & Random       & 0.579 & 21.825 \\
% Qwen     & Manual & Top-1 (Ours) & \textbf{0.650} & \textbf{22.088} \\
% \midrule
% DeepSeek & Noisy  & Random       & 0.503 & 21.626 \\
% DeepSeek & Noisy  & Top-1        & 0.612 & 21.750 \\
% DeepSeek & Manual & Random       & 0.604 & 21.974 \\
% DeepSeek & Manual & Top-1 (Ours) & \textbf{0.711} & \textbf{22.148} \\
% \bottomrule
% \end{tabular}
% \end{table}
% Required:
% \usepackage{booktabs}
% \usepackage{multirow}

\begin{table}[t]
\centering
\caption{Robustness of auxiliary prompts across different LLMs, prompt types, and selection strategies.}
\label{table_6}
\small
\setlength{\tabcolsep}{4.2pt}
\begin{tabular}{lllcc}
\toprule
\textbf{LLM} & \textbf{Prompt} & \textbf{Selection} & \textbf{TIFA} & \textbf{PickScore} \\
\midrule

\multirow{4}{*}{Qwen}
& \multirow{2}{*}{Noisy}
& Random       & 0.526 & 21.626 \\
&
& Top-1        & 0.550 & 21.740 \\
\cmidrule(lr){2-5}
&
\multirow{2}{*}{Manual}
& Random       & 0.579 & 21.825 \\
&
& Top-1 (Ours) & \textbf{0.650} & \textbf{22.088} \\

\midrule

\multirow{4}{*}{DeepSeek}
& \multirow{2}{*}{Noisy}
& Random       & 0.503 & 21.626 \\
&
& Top-1        & 0.612 & 21.750 \\
\cmidrule(lr){2-5}
&
\multirow{2}{*}{Manual}
& Random       & 0.604 & 21.974 \\
&
& Top-1 (Ours) & \textbf{0.711} & \textbf{22.148} \\

\bottomrule
\end{tabular}
\end{table}
\paragraph{Robustness of Auxiliary Prompts.} \label{section_4_3_4}
We evaluate the robustness of auxiliary prompt construction across different LLMs (Qwen and DeepSeek), prompt types (Noisy and Manual), and selection strategies (Random and Top-1). As shown in Table~\ref{table_6}, Manual prompts with Top-1 selection consistently achieve the best performance for both LLMs, reaching TIFA scores of 0.650 and 0.711 with Qwen and DeepSeek~\cite{liu2024deepseek}, respectively. Top-1 selection also improves over Random selection in all prompt settings, demonstrating that selecting semantically relevant auxiliary prompts reduces sensitivity to prompt variation. Moreover, the consistent trends across Qwen and DeepSeek indicate that \methodname{} is not subject to a specific LLM.

\section{Concluding Remarks} \label{section_5}
In this paper, we have proposed a training-free negation-aware generation method that disentangles entangled text representations through a bidirectional SVD-based decomposition-recomposition process between semantic-preserving and bias-suppressing subspaces. The proposed method not only suppresses undesirable biased semantics but also actively guides the representation toward the desired semantic direction while preserving the original semantic context and visual style of the input prompt. However, the proposed method still has limitations. Since semantic subspaces are estimated using auxiliary prompts, performance remains sensitive to prompt quality. Extending the method to more complex linguistic phenomena, such as relational and logical reasoning, remains an important direction for future work.

\section*{Acknowledgements} \label{sec_6}
This work was supported by IITP grant funded by the Korea government (MSIT) (No. RS-2020-II201361, Artificial Intelligence Graduate School Program (Yonsei University); No. RS-2022-II220113, Developing a Sustainable Collaborative Multi-modal Lifelong Learning Framework), and Air Force Defense Research Sciences Program funded by Air Force Office of Scientific Research.

%% file: section/2_supplementary.tex
\clearpage
\appendix

\section{Preliminaries}
\label{section_A}

Our method is built upon text-conditioned diffusion models. In text-conditioned diffusion models, an input prompt is transformed into latent representations through a text encoder $\tau_{\phi}$, and the generation process is guided by these text representations as conditioning information.

\paragraph{U-Net Based.}
Earlier U-Net based diffusion models~\cite{ho2020denoising,rombach2022high} mainly relied on a single text encoder for conditioning. Given an input prompt $p$, the prompt is transformed into a text embedding representation $x$, which is injected into cross-attention layers to guide image generation. Typically, CLIP~\cite{radford2021learning} is employed as the text encoder, allowing image features to selectively attend to semantic information from the text representation and enabling text-guided synthesis.

\paragraph{DiT Based.}
Recent transformer based diffusion models, including FLUX.1~\cite{labs2025flux1kontextflowmatching}, SD v3~\cite{esser2024scaling}, and SDXL~\cite{podell2024sdxl}, adopt multiple text encoders to simultaneously capture semantic information at different levels. Instead of relying on a single text representation, these models utilize complementary embeddings obtained from different encoders. Specifically, CLIP mainly captures global semantic and contextual information and produces a pooled representation $p_{\text{pool}}$, which is used in feature modulation to scale and shift intermediate representations before attention operations. In parallel, T5 provides token-wise contextual representations $p_{\text{emb}}$, preserving detailed prompt semantics and compositional information. The complete conditioning set $\mathcal{C}={p_{\text{pool}}, p_{\text{emb}}}$ is integrated into transformer blocks, where text and image features interact through independent query ($Q$), key ($K$), and value ($V$) transformations.

Such stronger text conditioning mechanisms significantly improve prompt fidelity by enhancing semantic alignment between text and generated images. However, despite these stronger conditioning mechanisms, generated images frequently retain undesired concepts under explicit negation prompts. In this work, we address this limitation by performing text embedding disentanglement directly within the text representation space prior to diffusion sampling.

\vfill\break

\section{Motivation Details}
\label{section_B}

\subsection{Profession Text Embedding Decomposition}
\label{section_B_1}

\begin{figure}[!h]
\centering
\includegraphics[width=\linewidth]{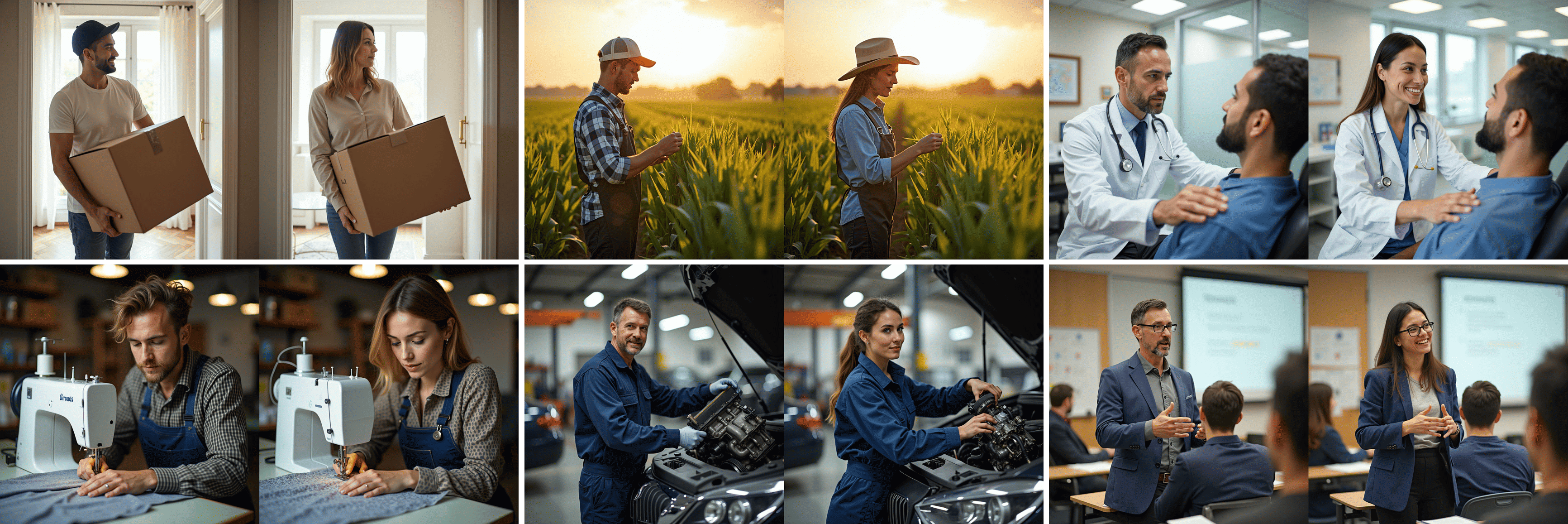}
\caption{Illustration of semantic entanglement of gender in profession-related prompt embeddings.}
\label{figure_9}
\end{figure}

Diffusion models~\cite{flux2024,podell2024sdxl} frequently exhibit semantic entanglement with demographic attributes in profession-related generation scenarios. For example, under profession-oriented prompts such as ``a dentist examining a patient,'' diffusion models often generate male individuals with a high frequency. This indicates that the profession embedding $\mathbf{x}_{\mathrm{dentist}}$ implicitly contains gender-related semantic priors rather than representing only occupation-specific semantics.

When the bias direction can be explicitly defined by a single reference vector (e.g., $\mathbf{x}_{\mathrm{man}}$), the entangled bias component within a target embedding can be isolated via vector projection:
\begin{equation}
\mathbf{x}^{\mathrm{bias}} = \mathcal{P}_{\mathbf{x}_{\mathrm{man}}}(\mathbf{x}_{\mathrm{dentist}}) 
\label{eq:gender_proj}
\end{equation}

We assume that the raw profession embedding (e.g., $\mathbf{x}_{\mathrm{dentist}}$) is linearly composed of its core occupational attributes (such as visual style or attire, denoted as $\mathbf{x}_{\mathrm{style}}$) and the entangled demographic bias ($\mathbf{x}_{\mathrm{bias}}$). This relationship can be formulated as follows:
\begin{equation}
\mathbf{x}_{\mathrm{dentist}} = \mathbf{x}_{\mathrm{style}} + \mathbf{x}^{\mathrm{bias}}.
\label{eq:gender_decompose}
\end{equation}
By exploiting this geometric formulation, the unbiased, profession-specific semantic component $\mathbf{x}_{\mathrm{style}}$ can be effectively extracted through vector subtraction:
\begin{equation}
\mathbf{x}^{\mathrm{style}} = \mathbf{x}_{\mathrm{dentist}} - \mathbf{x}_{\mathrm{bias}}.
\label{eq:gender_orth}
\end{equation}

Furthermore, we can synthesize a counterfactual embedding by substituting the original biased axis with an alternative demographic direction (e.g., $\mathbf{x}_{\mathrm{woman}}$):
\begin{equation}
\mathbf{x}'_{\mathrm{dentist}} = \mathbf{x}^{\mathrm{style}} + \mathbf{x}_{\mathrm{woman}}.
\label{eq:female_dentist}
\end{equation}
The qualitative results generated using the gender-modified counterfactual embedding are shown in Fig.~\ref{figure_9}.

\paragraph{General Token.}
However, extending this approach to general object-level concepts is challenging because semantic bias cannot be represented by a single explicit direction such as $x_{\mathrm{man}}$ or $x_{\mathrm{woman}}$. In practice, semantic bias is distributed across multiple correlated patterns. For example, in the prompt ``bakery with no bread,'' bread-related semantics are entangled with associated concepts such as cakes, trays, and food structures. Therefore, instead of relying on a single semantic direction, we construct semantic subspaces using multiple auxiliary prompts. Shared semantic components are extracted from related prompt embeddings, enabling more stable estimation of semantic information for negation-aware embedding construction.

\subsection{Padding Analysis}
\label{section_B_2}

\begin{table}[t]
\centering
\caption{Similarity analysis between text embeddings and padding embeddings across different architectures.}
\label{table_7}
\tiny
\setlength{\tabcolsep}{8pt}
\renewcommand{\arraystretch}{1.12}
\begin{adjustbox}{width=\linewidth}
\begin{tabular}{lcc}
\toprule
Arch & Max Similarity & Mean Similarity \\
\midrule
DiT & 0.142 & 0.024 \\
U-Net & 0.992 & 0.403 \\
\bottomrule
\end{tabular}
\end{adjustbox}
\end{table}

SupEOT performs negation-aware generation by suppressing semantic components within padding embeddings. The method assumes that padding regions in text embeddings still contain meaningful semantic information, and that removing such components can improve negation performance. However, unlike U-Net architectures, recent DiT models exhibit relatively cleaner padding embeddings with minimal semantic information. As shown in Tab.~\ref{table_7}, U-Net models demonstrate relatively high similarity between text and padding embeddings, suggesting that padding regions retain meaningful semantic content. In contrast, DiT architectures show extremely low similarity between text and padding embeddings, indicating that padding regions are largely disentangled from text semantics. These observations suggest that the padding embedding manipulation strategy proposed in SupEOT may not generalize effectively across different architectures.

\vfill\break
\section{Theoretical Analysis} \label{section_C}

This section formalizes why Uni-SVD fails to preserve stylistic information under negation, and how the proposed Bi-SVD overcomes this failure through cross-subspace recovery.

\begin{definition}[Subspace Structure]
\label{def:subspace}
Let $\mathcal{P}_{\mathcal{S}}$ and $\mathcal{P}_{\mathcal{B}}$ denote the orthogonal projectors onto the semantic-preserving subspace $\mathcal{S}$ and the bias subspace $\mathcal{B}$, respectively. The principal angles $\theta_{1}\!\le\!\cdots\!\le\!\theta_{r}$ between $\mathcal{S}$ and $\mathcal{B}$ satisfy $\cos\theta_{i}=\sigma_{i}({V_S^{(k)}}^{\!\top}\!V_B^{(k)})$, where $\sigma_i(\cdot)$ denotes the $i$-th largest singular value. A smaller $\theta_1$ indicates stronger entanglement, and $\theta_1$ is non-increasing in $k$ as larger subspaces admit more overlapping directions.
\end{definition}

\begin{definition}[Embedding Decomposition]
\label{def:decomp}
A negation prompt embedding $x_n\in\mathbb{R}^{d}$ is decomposed as
\begin{equation}
x_n = x_n^{\mathrm{sem}} + x_n^{\mathrm{bias}} + x_n^{\mathrm{style},B} + x_n^{\mathrm{style},S},
\label{eq:decomp}
\end{equation}
under the following assumptions:
\begin{itemize}[itemsep=2pt]
\item[\textbf{(A1)}]
  $x_n^{\mathrm{sem}}\!\in\!\mathcal{S}$ with
  $\mathcal{P}_{\mathcal{B}}(x_n^{\mathrm{sem}})\!=\!0$.
\item[\textbf{(A2)}]
  $x_n^{\mathrm{bias}}\!\in\!\mathcal{B}$ with $\mathcal{P}_{\mathcal{S}}(x_n^{\mathrm{bias}})\!=\!0$.
\item[\textbf{(A3)}]
  $x_n^{\mathrm{style},B}\!\in\!\mathcal{B}$ with $\mathcal{P}_{\mathcal{S}}(x_n^{\mathrm{style},B})\!=\!\eta_b\!\ne\!0$, representing bias-entangled style.
\item[\textbf{(A4)}]
  $x_n^{\mathrm{style},S}\!\in\!\mathcal{S}$ with $\mathcal{P}_{\mathcal{B}}(x_n^{\mathrm{style},S})\!=\!\eta_s\!\ne\!0$, representing semantic-aligned style.
\end{itemize}
\end{definition}

\begin{definition}[Operators]
\label{def:ops}
\begin{align*}
\mathcal{T}_{\mathrm{uni}} &= \mathbb{I}- \mathcal{P}_{\mathcal{B}}, \\
\mathcal{P}^{\mathrm{sem}} &= (\mathbb{I}\!-\!\mathcal{P}_{\mathcal{B}})\,\mathcal{P}_{\mathcal{S}}, \\
\mathcal{P}^{\mathrm{style}} &= \mathcal{P}_{\mathcal{S}}\,(\mathbb{I}\!-\!\mathcal{P}_{\mathcal{B}}), \\
\mathcal{T}_{\mathrm{bi}}
  &= \mathcal{P}^{\mathrm{sem}} + \mathcal{P}^{\mathrm{style}} \notag\\
  &= 2\mathcal{P}_{\mathcal{S}} 
   - \mathcal{P}_{\mathcal{B}}\mathcal{P}_{\mathcal{S}} 
   - \mathcal{P}_{\mathcal{S}}\mathcal{P}_{\mathcal{B}}.
   \label{eq:T_bi}
\end{align*}
\end{definition}

We define $\Delta_{\mathrm{style}}\!=\!\mathcal{P}_{\mathcal{S}}\mathcal{P}^{\mathrm{sem}}(x_n^{\mathrm{style},B})$ 
and $\Delta_{\mathrm{bias}}\!=\!\mathcal{P}_{\mathcal{B}}\mathcal{P}^{\mathrm{style}}(x_n^{\mathrm{style},S})$ 
to denote the recovered style and residual bias terms, respectively.

\begin{figure*}[t]
    \centering
    \includegraphics[width=\linewidth]{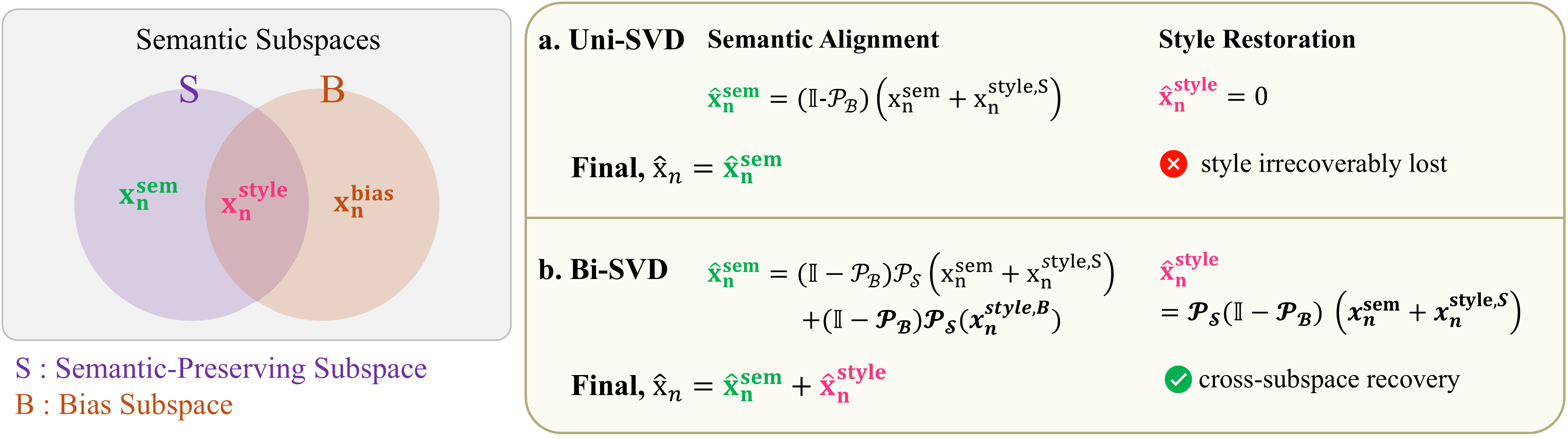}
    \caption{Comparison between Uni-SVD and Bi-SVD for negation embedding decomposition. Unlike Uni-SVD, Bi-SVD selectively preserves contextual semantics while suppressing bias-related representations.}
    \label{figure_10}
\end{figure*}

% ────────────────────────────────────────────────────────────
\subsection{Limitations of Uni-SVD}
% ────────────────────────────────────────────────────────────

\newtheorem*{theorem_uni}{Theorem~\ref{theorem_1}}
\begin{theorem_uni}[Uni-SVD: Style Degradation]
\label{thm:uni}
Under \textbf{(A1)}--\textbf{(A4)}, $\mathcal{T}_{\mathrm{uni}}=\mathbb{I}-\mathcal{P}_{\mathcal{B}}$ yields:
\begin{equation}
\mathcal{T}_{\mathrm{uni}}(x_n^{\mathrm{bias}}) = 0,
\qquad
\mathcal{T}_{\mathrm{uni}}(x_n^{\mathrm{style},B}) = 0.
\label{eq:uni_results}
\end{equation}
\end{theorem_uni}

\begin{proof}
\begin{align*}
&\text{For } x_n^{\mathrm{bias}}\text{: by \textbf{(A2)}, }
 \mathcal{P}_{\mathcal{B}}(x_n^{\mathrm{bias}})=x_n^{\mathrm{bias}}\text{, thus}\\
&\mathcal{T}_{\mathrm{uni}}(x_n^{\mathrm{bias}})
= (\mathbb{I}-\mathcal{P}_{\mathcal{B}})(x_n^{\mathrm{bias}}) = 0 \\[6pt]
&\text{For } x_n^{\mathrm{style,B}}\text{: by \textbf{(A3)}, }
 \mathcal{P}_{\mathcal{B}}(x_n^{\mathrm{style,B}})=x_n^{\mathrm{style,B}}\text{, thus}\\
&\mathcal{T}_{\mathrm{uni}}(x_n^{\mathrm{style,B}})
= (\mathbb{I}-\mathcal{P}_{\mathcal{B}})(x_n^{\mathrm{style,B}}) = 0
\end{align*}
\end{proof}

Because $\mathcal{T}_{\mathrm{uni}}$ treats the entire subspace $\mathcal{B}$ uniformly, it cannot distinguish bias from style entangled within $\mathcal{B}$. Consequently, the stylistic component $x_n^{\mathrm{style}}$ is irrecoverably lost together with the bias, directly causing the visual fidelity collapse observed in Fig.~\ref{figure_2}(c). %observed in Fig.~\ref{figure_1}(a).
Stronger entanglement (larger $\|x_n^{\mathrm{style}}\|$) intensifies this loss,
and no choice of suppression strength within the uni-directional framework
can simultaneously remove bias and preserve $x_n^{\mathrm{style}}$.

% ────────────────────────────────────────────────────────────
\subsection{Properties of Bi-SVD}
% ────────────────────────────────────────────────────────────

\newtheorem*{theorem_bi}{Theorem~\ref{theorem_2}}
\begin{theorem_bi}[Bi-SVD: Bounded Bias and Style Recovery]
\label{thm:bi}
Under \textbf{(A1)}--\textbf{(A3)}, the bidirectional operator $\mathcal{T}_{\mathrm{bi}}$
satisfies:

\noindent\textbf{(a)} Bias suppression:
\begin{equation}
\mathcal{T}_{\mathrm{bi}}(x_n^{\mathrm{bias}}) = 0.
\label{eq:bi_bias}
\end{equation}

\noindent\textbf{(b)} Bounded bias residual:
\begin{equation}
\mathcal{P}_{\mathcal{B}}(\mathcal{T}_{\mathrm{bi}}\,(x_n)) = \Delta_{\mathrm{bias}},
\label{eq:bi_residual}
\end{equation}
where $\|\Delta_{\mathrm{bias}}\| \le \cos\theta_{1}\,\|(\mathbb{I}\!-\!\mathcal{P}_{\mathcal{B}})\,x_n\|$.

The residual $\Delta_{\mathrm{bias}}$ arises solely from $\eta_s$ through
a three-step attenuation chain
$\eta_s \!\to\! \mathcal{P}_{\mathcal{S}}(\eta_s) \!\to\! \mathcal{P}_{\mathcal{B}}\mathcal{P}_{\mathcal{S}}(\eta_s)$,
each step reducing energy by $\cos\theta_1$.
It becomes negligible when auxiliary prompts yield well-separated subspaces,
and vanishes exactly when $\eta_s=0$.

\noindent\textbf{(c)} Cross-subspace style recovery:
\begin{equation}
\mathcal{T}_{\mathrm{bi}}(x_n^{\mathrm{style},B}) = \mathcal{P}^{\mathrm{sem}}(x_n^{\mathrm{style},B}) \neq 0,
\label{eq:bi_recovery_B}
\end{equation}
\begin{equation}
\mathcal{P}^{\mathrm{style}}(x_n^{\mathrm{style},S}) \neq 0.
\label{eq:bi_recovery_S}
\end{equation}

\noindent\textbf{(d)} Additional semantic content vs.\ Uni-SVD:
\begin{equation}
\mathcal{P}_{\mathcal{S}}(\mathcal{T}_{\mathrm{bi}}(x_n))
= 2\,\mathcal{P}_{\mathcal{S}}(\mathcal{T}_{\mathrm{uni}}(x_n)) + \Delta_{\mathrm{style}}.
\label{eq:bi_semantic}
\end{equation}
While $\mathcal{T}_{\mathrm{uni}}$ unconditionally annihilates $x_n^{\mathrm{style}}$
(Theorem 1), $\mathcal{T}_{\mathrm{bi}}$ rescues its semantic-aligned content
via cross-subspace recovery.
The total preserved semantic content of Bi-SVD strictly exceeds that of Uni-SVD
by $\Delta_{\mathrm{style}}$, which represents bias-free stylistic knowledge
extracted from $\mathcal{B}$ that is structurally inaccessible to any
uni-directional operator.
\end{theorem_bi}

\begin{proof}
\noindent\textbf{(a)} By \textbf{(A2)}, 
$\mathcal{P}_{\mathcal{S}}(x_n^{\mathrm{bias}})\!=\!0$ 
and $(\mathbb{I}-\mathcal{P}_{\mathcal{B}})(x_n^{\mathrm{bias}})\!=\!0$. Thus:
\begin{align*}
\mathcal{P}^{\mathrm{sem}}(x_n^{\mathrm{bias}}) 
&= (\mathbb{I}-\mathcal{P}_{\mathcal{B}})\mathcal{P}_{\mathcal{S}}(x_n^{\mathrm{bias}}) = 0 \\
\mathcal{P}^{\mathrm{style}}(x_n^{\mathrm{bias}}) 
&= \mathcal{P}_{\mathcal{S}}(\mathbb{I}-\mathcal{P}_{\mathcal{B}})(x_n^{\mathrm{bias}}) = 0 \\
\therefore\,\mathcal{T}_{\mathrm{bi}}(x_n^{\mathrm{bias}}) 
&= \mathcal{P}^{\mathrm{sem}}(x_n^{\mathrm{bias}}) 
 + \mathcal{P}^{\mathrm{style}}(x_n^{\mathrm{bias}}) = 0
\end{align*}

\noindent\textbf{(c)} For $x_n^{\mathrm{style,B}}$: by \textbf{(A3)}, 
$\mathcal{P}_{\mathcal{S}}(x_n^{\mathrm{style,B}})\!=\!\eta_b$ and $(\mathbb{I}-\mathcal{P}_{\mathcal{B}})(x_n^{\mathrm{style,B}})\!=\!0$. Thus:
\begin{align*}
\mathcal{P}^{\mathrm{style}}(x_n^{\mathrm{style,B}}) 
&= \mathcal{P}_{\mathcal{S}}(\mathbb{I}-\mathcal{P}_{\mathcal{B}})(x_n^{\mathrm{style,B}}) = 0 \\
\mathcal{P}^{\mathrm{sem}}(x_n^{\mathrm{style,B}}) 
&= (\mathbb{I}-\mathcal{P}_{\mathcal{B}})\mathcal{P}_{\mathcal{S}}(x_n^{\mathrm{style,B}}) \\
&= (\mathbb{I}-\mathcal{P}_{\mathcal{B}})\,\eta_b \neq 0\\
\therefore\,\mathcal{T}_{\mathrm{bi}}(x_n^{\mathrm{style,B}}) 
&= (\mathbb{I}-\mathcal{P}_{\mathcal{B}})\,\eta_b \neq 0
\end{align*}

\noindent For $x_n^{\mathrm{style,S}}$: by \textbf{(A4)}, 
$\mathcal{P}_{\mathcal{B}}(x_n^{\mathrm{style,S}})\!=\!\eta_s$ and 
$\mathcal{P}_{\mathcal{S}}(x_n^{\mathrm{style,S}})\!=\!x_n^{\mathrm{style,S}}$. Thus:
\begin{align*}
% \mathcal{P}^{\mathrm{sem}}(x_n^{\mathrm{style,S}}) 
% &= (\mathbb{I}-\mathcal{P}_{\mathcal{B}})\mathcal{P}_{\mathcal{S}}(x_n^{\mathrm{style,S}}) \\
% &= (\mathbb{I}-\mathcal{P}_{\mathcal{B}})\,x_n^{\mathrm{style,S}} 
%  = x_n^{\mathrm{style,S}} - \eta_s \\
\mathcal{P}^{\mathrm{style}}(x_n^{\mathrm{style,S}}) 
&= \mathcal{P}_{\mathcal{S}}(\mathbb{I}-\mathcal{P}_{\mathcal{B}})(x_n^{\mathrm{style,S}}) \\
&= \mathcal{P}_{\mathcal{S}}(x_n^{\mathrm{style,S}}-\eta_s) \\
&= x_n^{\mathrm{style,S}} - \mathcal{P}_{\mathcal{S}}(\eta_s) \neq 0
\end{align*}

\noindent\textbf{(b)} Since $\mathcal{P}_{\mathcal{B}}(\mathbb{I}-\mathcal{P}_{\mathcal{B}})=0$,
we have $\mathcal{P}_{\mathcal{B}}\mathcal{P}^{\mathrm{sem}}=0$. Thus:
\vspace{-3pt}
\begin{align*}
&\mathcal{P}_{\mathcal{B}}(\mathcal{T}_{\mathrm{bi}}(x_n)) \\
&=\!\mathcal{P}_{\mathcal{B}}(\mathcal{T}_{\mathrm{bi}}(x_n^{\mathrm{style,S}}\!+\!x_n^{\mathrm{style,B}}))
  \quad \text{(by Thm.~2a, (A1))} \\
&=\!\mathcal{P}_{\mathcal{B}}(\mathcal{T}_{\mathrm{bi}}(x_n^{\mathrm{style,S}})
  \!+\!\mathcal{P}^{\mathrm{sem}}(x_n^{\mathrm{style,B}}))
  \, \text{(by Thm.~2c)} \\
&=\!\mathcal{P}_{\mathcal{B}}(\mathcal{P}^{\mathrm{style}}(x_n^{\mathrm{style,S}}))
  \quad (\text{since } \mathcal{P}_{\mathcal{B}}\mathcal{P}^{\mathrm{sem}}=0) \\
&=\!\Delta_{\mathrm{bias}}.
\end{align*}
The bound follows from $\|\mathcal{P}_{\mathcal{B}}\mathcal{P}_{\mathcal{S}}\|_{\mathrm{op}}=\cos\theta_1$.

\noindent\textbf{(d)} First, by Theorem~1, 
$\mathcal{T}_{\mathrm{uni}}(x_n^{\mathrm{bias}})\!=\!\mathcal{T}_{\mathrm{uni}}(x_n^{\mathrm{style,B}})\!=\!0$:
\begin{align*}
\mathcal{P}_{\mathcal{S}}(\mathcal{T}_{\mathrm{uni}}(x_n))
&= \mathcal{P}^{\mathrm{style}}(x_n^{\mathrm{sem}}) 
 + \mathcal{P}^{\mathrm{style}}(x_n^{\mathrm{style,S}})
\end{align*}
For $\mathcal{T}_{\mathrm{bi}}$, applying 
$\mathcal{T}_{\mathrm{bi}} = 2\mathcal{P}_{\mathcal{S}} 
- \mathcal{P}_{\mathcal{B}}\mathcal{P}_{\mathcal{S}} 
- \mathcal{P}_{\mathcal{S}}\mathcal{P}_{\mathcal{B}}$:
\begin{align*}
\mathcal{P}_{\mathcal{S}}(\mathcal{T}_{\mathrm{bi}}(x_n^{\mathrm{sem}}))
&= 2\mathcal{P}_{\mathcal{S}}x_n^{\mathrm{sem}} 
 - 2\mathcal{P}_{\mathcal{S}}\mathcal{P}_{\mathcal{B}}(x_n^{\mathrm{sem}}) \\
&= 2\mathcal{P}^{\mathrm{style}}(x_n^{\mathrm{sem}})
\end{align*}
\begin{align*}
\mathcal{P}_{\mathcal{S}}(\mathcal{T}_{\mathrm{bi}}(x_n^{\mathrm{style,S}}))
&= \mathcal{P}_{\mathcal{S}}(2x_n^{\mathrm{style,S}} 
 - \eta_s - \mathcal{P}_{\mathcal{S}}(\eta_s)) \\
&= 2\mathcal{P}^{\mathrm{style}}(x_n^{\mathrm{style,S}})
\end{align*}
\begin{align*}
&\mathcal{P}_{\mathcal{S}}(\mathcal{T}_{\mathrm{bi}}(x_n^{\mathrm{style,B}})) \\
&= \mathcal{P}_{\mathcal{S}}\mathcal{P}^{\mathrm{sem}}(x_n^{\mathrm{style,B}})
 + \mathcal{P}_{\mathcal{S}}
   \underbrace{(\mathbb{I}-\mathcal{P}_{\mathcal{B}})(x_n^{\mathrm{style,B}})}_{=\,0\ \text{(by (A3))}} \\
&= \Delta_{\mathrm{style}}
\end{align*}
\noindent Combining with $\mathcal{P}_{\mathcal{S}}(\mathcal{T}_{\mathrm{bi}}(x_n^{\mathrm{bias}}))=0$
(by Theorem~2a), we obtain:
\[
\mathcal{P}_{\mathcal{S}}(\mathcal{T}_{\mathrm{bi}}(x_n))
= 2\,\mathcal{P}_{\mathcal{S}}(\mathcal{T}_{\mathrm{uni}}(x_n)) + \Delta_{\mathrm{style}}.
\]
\end{proof}

\vfill\break

\section{Method Detail}
\label{section_D}

\subsection{Auxiliary Prompts}
\label{section_D_1}

Auxiliary prompts are generated using Qwen2.5-Instruct~\cite{bai2025qwen25vl}. Given a negation prompt $p_n$, we construct two auxiliary prompt sets: a biased semantic set $A_B$ that strongly represents the target concept and a semantic set $A_S$ that preserves the original semantics while avoiding explicit emphasis on the target concept.

To improve prompt quality and semantic consistency, we employ Sentence-BERT ~\cite{reimers2019sentencebert} to compute sentence-level semantic similarities between generated prompts and target semantic concepts.
In each round, $C=20$ candidate prompts are generated and the one maximizing semantic similarity to its corresponding target concept is selected. In addition, prompt length is constrained to maintain concise and visually descriptive representations. Detailed LLM instructions and prompt selection strategy are summarized in Fig.~\ref{figure_instruction} and Alg.~\ref{algorithm_1}, respectively.

\begin{algorithm}[h]
\caption{Iterative Auxiliary Prompt Selection Strategy}
\label{algorithm_1}
%\normalsize
\small
\begin{algorithmic}[1]
\Require Negation prompt $p_n$, biased concept $p_b$, preserved semantic concept $p_s$, number of rounds $N$, candidates per round $C$
\Ensure Biased prompt set $A_B$, Semantic prompt set $A_S$
\vspace{0pt}
\State Initialize prompt sets: $A_B \gets \emptyset$, $A_S \gets \emptyset$.
\For{$i = 1$ to $N$}
    \State Generate $C$ candidate prompts for round $i$ via LLM:
    \[
    A_B^i = \{a_b^{i,1}, a_b^{i,2}, \dots, a_b^{i,C}\} \sim \text{LLM}(p_b)
    \]
    \[
    A_S^i = \{a_s^{i,1}, a_s^{i,2}, \dots, a_s^{i,C}\} \sim \text{LLM}(p_s)
    \]
    \State Select the optimal prompt maximizing semantic similarity via text encoder $\mathcal{E}$:
    \[
    a_b^i = \underset{a \in A_B^i}{\arg\max} \, \text{Sim}(\mathcal{E}(a), \mathcal{E}(p_b))
    \]
    \[
    a_s^i = \underset{a \in A_S^i}{\arg\max} \, \text{Sim}(\mathcal{E}(a), \mathcal{E}(p_s))
    \]
    \State Accumulate selected prompts: $A_B \gets A_B \cup \{a_b^i\}$, \quad $A_S \gets A_S \cup \{a_s^i\}$
\EndFor
\vspace{1pt}
\State \Return $A_B, A_S$
\end{algorithmic}
\end{algorithm}

\begin{figure}[!h]
% \vspace{-125pt}
\centering
\begin{tcolorbox}[
    colback=boxbg,
    colframe=boxframe,
    arc=2mm,
    boxrule=0.8pt,
    left=6pt,
    right=6pt,
    top=6pt,
    bottom=6pt,
    width=\linewidth
]
\small\sffamily\raggedright
{\bfseries \textcolor{darkgray}{\underline{Instruction}}}\\
\vspace{2pt}

You are a model responsible for constructing auxiliary prompts for semantic subspace decomposition. 
Given a negation and semantic concept, generate two types of auxiliary prompt sets. \\
%%%%%%%%%%%%%%%%%%%%%%%%%%%%%%%%%%%%%%%%%%%%%%%%%%%%%%%%%%%%%%%%%%%%%%%%%%%%%%%%%%%%%%%%%%%%%%%%%%%%%%%%%%%%%%%%
\vspace{12pt}
%%%%%%%%%%%%%%%%%%%%%%%%%%%%%%%%%%%%%%%%%%%%%%%%%%%%%%%%%%%%%%%%%%%%%%%%%%%%%%%%%%%%%%%%%%%%%%%%%%%%%%%%%%%%%%%%
\vspace{6pt}
{\bfseries \textcolor{darkgray}{\underline{Example Input}}}\\
\vspace{2pt}
\textcolor{orange!85!black}{\textbf{[BIASED]}}\\
bread

\vspace{4pt}

\textcolor{green!70!black}{\textbf{[SEMANTIC]}}\\
an empty shop interior with clean shelves
%%%%%%%%%%%%%%%%%%%%%%%%%%%%%%%%%%%%%%%%%%%%%%%%%%%%%%%%%%%%%%%%%%%%%%%%%%%%%%%%%%%%%%%%%%%%%%%%%%%%%%%%%%%%%%%%
\vspace{12pt}
%%%%%%%%%%%%%%%%%%%%%%%%%%%%%%%%%%%%%%%%%%%%%%%%%%%%%%%%%%%%%%%%%%%%%%%%%%%%%%%%%%%%%%%%%%%%%%%%%%%%%%%%%%%%%%%%
\vspace{6pt}

{\bfseries \textcolor{darkgray}{\underline{Definition}}}\\
\vspace{2pt}
\textcolor{orange!85!black}{\textbf{[BIASED]}}\\
Prompts that strongly contain or imply the target negation concept. 
These prompts are used to construct the bias subspace $\mathcal{B}$.

\vspace{4pt}

\textcolor{green!70!black}{\textbf{[SEMANTIC]}}\\
Prompts that avoid the target negation concept while preserving the original semantic. 
These prompts are used to construct the semantic-preserving subspace $\mathcal{S}$.
%%%%%%%%%%%%%%%%%%%%%%%%%%%%%%%%%%%%%%%%%%%%%%%%%%%%%%%%%%%%%%%%%%%%%%%%%%%%%%%%%%%%%%%%%%%%%%%%%%%%%%%%%%%%%%%%
\vspace{12pt}
%%%%%%%%%%%%%%%%%%%%%%%%%%%%%%%%%%%%%%%%%%%%%%%%%%%%%%%%%%%%%%%%%%%%%%%%%%%%%%%%%%%%%%%%%%%%%%%%%%%%%%%%%%%%%%%%
\vspace{6pt}
{\bfseries \textcolor{darkgray}{\underline{Generation Guidelines}}}\\
\vspace{2pt}
\textcolor{orange!85!black}{\textbf{[BIASED]}}\\
- Strongly activate the target negation concept.\\
- Describe scenes where the base semantic naturally contains the target negation concept.\\
- Use concrete and visually descriptive language.\\
- Preserve consistency with the original semantic concept.

\vspace{4pt}

\textcolor{green!70!black}{\textbf{[SEMANTIC]}}\\
- Avoid directly mentioning the target negation concept.\\
- Preserve the original semantic.\\
- Focus on visually similar, clean, empty, minimal, or uncluttered scenes.\\
- Do not explicitly use negation expressions such as ``without'', ``no'', ``missing'', or ``lacks''.

\vspace{6pt}
\vspace{6pt}

\vspace{6pt}
{\bfseries \textcolor{darkgray}{\underline{Example Output}}}\\
\vspace{2pt}

\textcolor{orange!85!black}{\textbf{[BIASED]}}\\
- \textcolor{orange!85!black}{A bakery filled with freshly baked bread on display shelves.}\\
- \textcolor{orange!85!black}{A bakery counter stacked with warm loaves of bread.}\\
- \textcolor{orange!85!black}{Fresh bread arranged inside a bakery display case.}

\vspace{4pt}

\textcolor{green!70!black}{\textbf{[SEMANTIC]}}\\
- \textcolor{green!70!black}{A clean bakery interior with organized display shelves.}\\
- \textcolor{green!70!black}{A quiet bakery shop with minimal interior structure.}\\
- \textcolor{green!70!black}{A bakery counter with neatly arranged empty display racks.}

\end{tcolorbox}
\caption{LLM Instructions for Auxiliary Prompt Construction}
\label{figure_instruction}
\vspace{-8pt}
\end{figure}

\vfill\break

\subsection{Bidirectional Projection Algorithm}
\label{section_D_2}
Given a negation prompt $p_n$, the proposed method adaptively refines the text embedding by geometrically disentangling the pure semantics and residual styles from implicit biases, as summarized in Alg.~\ref{algorithm_2}.

%%%%%%%%%%%%%%%%%%%%%%%%%%%%%%%%%%%%%%%%%%%%%%%%%%%%%%%%%%%%%%%%%%%%%%%%%%%%%%%%%%%%%%%%%%%%%%%%%%%%%%%%%%%%%%%%%%%%%%%%%%%%
\begin{algorithm}[!h]
\caption{Negation-aware Text Embedding via Bi-SVD}
\label{algorithm_2}
%\normalsize
\small
\begin{algorithmic}[1]
\Require Negation prompt $p_n$ (embedding $x_n$), auxiliary embeddings $X_S, X_B \in \mathbb{R}^{N \times d}$
\Ensure Negation-aware embedding $\hat{x}_n$
\vspace{0pt}
\Statex \textbf{// Step 1: Subspace Estimation via Mean-Centered SVD}
\State Calculate structural means: 
\[
\mu_{\mathcal{S}} = \mathrm{Mean}(X_S), \qquad \mu_{\mathcal{B}} = \mathrm{Mean}(X_B)
\]
\State Perform SVD on mean-centered matrices (Eq.~\ref{equation_1}):
\[
\begin{aligned}
X_S - \mu_{\mathcal{S}}
&= U_S \Sigma_S V_S^\top, \\
X_B - \mu_{\mathcal{B}}
&= U_B \Sigma_B V_B^\top.
\end{aligned}
\]

% \State Perform SVD on mean-centered matrices (Eq.~\ref{equation_1}):
% %%%%%%%%%%%%%%%% 여기같은 곳이 좀 줄을 나가버리네? 이런거 처리할 수 있어? 
% \[
% X_S - \mu_{\mathcal{S}} = U_S \Sigma_S V_S^\top, \qquad X_B - \mu_{\mathcal{B}} = U_B \Sigma_B V_B^\top
% \]
\State Define affine projection operators with top-$k$ bases $V_S^{(k)}$ and $V_B^{(k)}$:
\[
\mathcal{P}_{\mathcal{S}}(x) = \mu_{\mathcal{S}} + V_S^{(k)} {V_S^{(k)}}^\top (x - \mu_{\mathcal{S}})
\]
\[
\mathcal{P}_{\mathcal{B}}(x) = \mu_{\mathcal{B}} + V_B^{(k)} {V_B^{(k)}}^\top (x - \mu_{\mathcal{B}})
\]
\vspace{1pt}
\Statex \textbf{// Step 2: Bidirectional Projection \& Residual Bias Removal}
\State Extract the bias-reduced semantic representation (Eq.~\ref{equation_2}):
\[
\hat{x}_n^{\mathrm{sem}} = (\mathbb{I} - \mathcal{P}_{\mathcal{B}})\mathcal{P}_{\mathcal{S}}(x_n)
\]
%\vspace{1pt}
\Statex \textbf{// Step 3: Residual Style Restoration \& Final Recomposition}
\State Restore residual style details onto $\mathcal{S}$ (Eq.~\ref{equation_3}):
\[
\hat{x}_n^{\mathrm{style}} = \mathcal{P}_{\mathcal{S}}(\mathbb{I} - \mathcal{P}_{\mathcal{B}})(x_n)
\]
\State Recompose the final negation-aware text embedding (Eq.~\ref{equation_4}):
\[
\hat{x}_n = \hat{x}_n^{\mathrm{sem}} + \hat{x}_n^{\mathrm{style}}
\]
\State \Return $\hat{x}_n$
\end{algorithmic}
\end{algorithm}
%%%%%%%%%%%%%%%%%%%%%%%%%%%%%%%%%%%%%%%%%%%%%%%%%%%%%%%%%%%%%%%%%%%%%%%%%%%%%%%%%%%%%%%%%%%%%%%%%%%%%%%%%%%%%%%%%%%%%%%%%%%%

\vfill\break

%%%%%%%%%%%%%%%%%%%%%%%%%%%%%%%%%%%%%%%%%%%%%%%%%%%%%%%%%%%%%%%%%%%%%%%%%%%%%%%%%%%%%%%%%%%%%%%%%%%%%%%%%%%%%%%%%%%%%%%%%%%%%%%%%%%%%%%%%%%%%%%%%%%%%%%%%%%%%%%%%%%%%%%%%%%%%%%%%%%%%%%%%%%%%%%%%%%%%%%%%%%%%%%%%%%%%%%%%%%%%%%%%%%%%%%%%%%%%%%%%%%%%%%%%%%%%%%%%%%%%%%%%%%%%%%%%%%%%%%%%%%%%%%%%%%%%%%%%%%%%%%%%%%%%%%%%%%%%%%%%%%%%%%

\clearpage
\section{Experimental Details}
\label{section_F}

\begin{figure}[!t]
\centering
\vspace{-5pt}
\includegraphics[width=\linewidth]{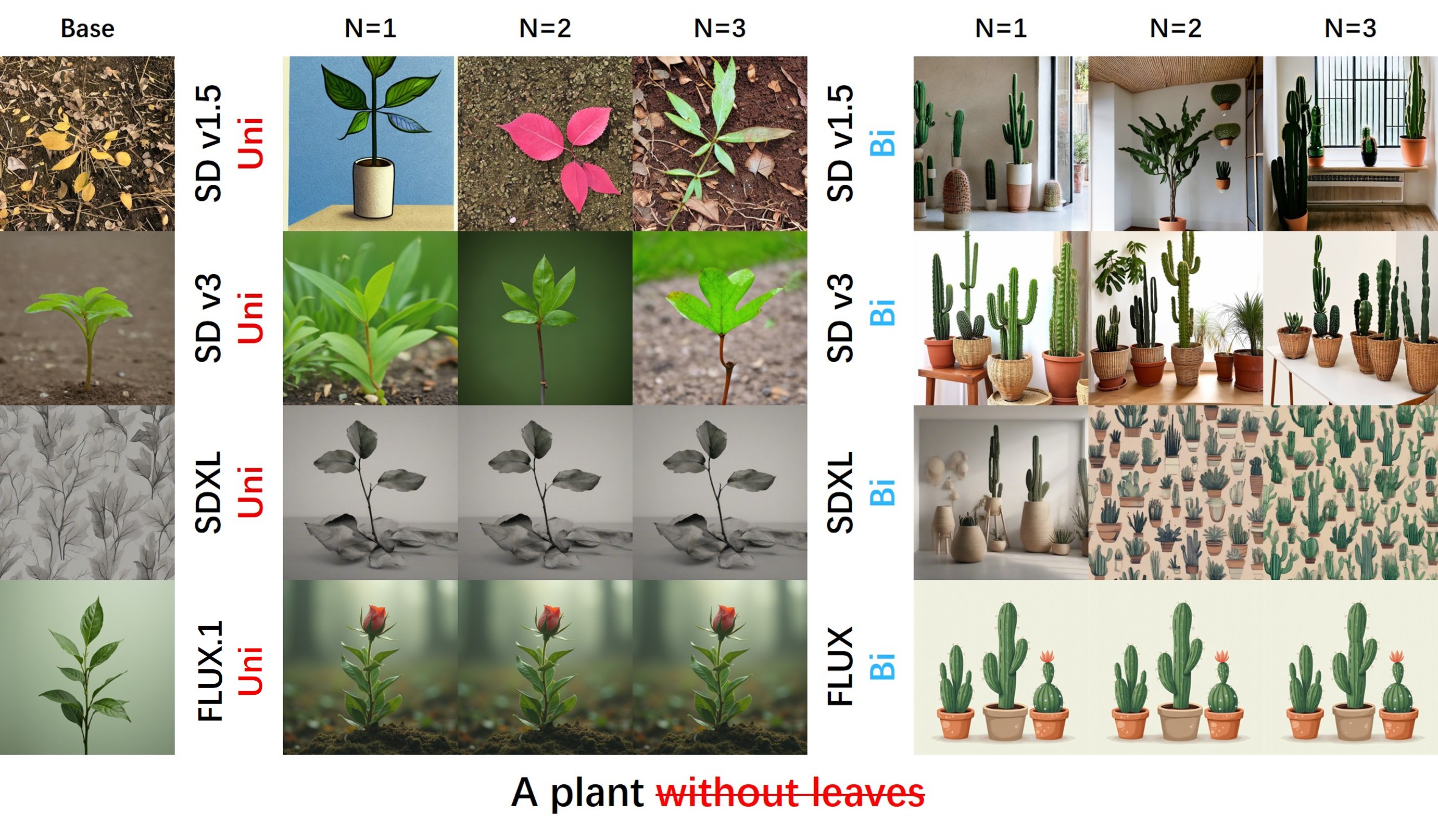}
\caption{Qualitative comparison of Base, Uni-SVD, and Bi-SVD under varying numbers of auxiliary prompts.}
\label{figure_11}
\end{figure}

\subsection{Experimental Setup}
\label{section_F_exp}
All experiments were conducted on an NVIDIA A100 GPU. To ensure a fair comparison, the number of inference steps was kept identical to the default settings of each corresponding base model, and a consistent random seed was used across all methods. Furthermore, all models utilized in this study were used in strict compliance with their respective licenses.

\subsection{Uni-SVD vs Bi-SVD}
\label{section_F_1}

Fig.~\ref{figure_11} provides additional qualitative results supporting the analysis in \S~\ref{section_4_3} regarding the effectiveness of Bi-SVD. U-Net based models such as SD v1.4 and SD v1.5 exhibit unstable removal of target negation semantics when \textbf{Uni-SVD} is applied. As discussed in prior work such as SupEOT, the generation process of diffusion models is influenced not only by text embeddings but also by multiple factors including padding tokens and other latent interactions. Consequently, controlling only a single semantic direction through \textbf{Uni-SVD} becomes insufficient for resolving such complex semantic entanglement in a stable manner.

In contrast, under FLUX.1, \textbf{Uni-SVD} gradually estimates more meaningful semantic directions as the number of auxiliary prompts increases, resulting in improved semantic disentanglement performance. This behavior becomes particularly evident in FLUX.1 because, as reported in previous studies~\cite{dalva2024fluxspace}, the embedding representations of FLUX.1 exhibit relatively disentangled properties compared to other diffusion models, leading to more clearly separated semantic structures. As a result, dominant semantic components can be estimated more reliably even through simple unidirectional semantic decomposition, allowing \textbf{Uni-SVD} to achieve a certain level of semantic manipulation performance.

However, for another Transformer based architecture, SD v3, \textbf{Uni-SVD} still fails to reliably preserve negation behavior despite its stronger text understanding capability. This observation suggests that the effectiveness of \textbf{Uni-SVD} is not determined solely by the backbone architecture itself. Unless the underlying embedding representations already possess disentangled semantic structures similar to FLUX.1, \textbf{Uni-SVD} remains limited in handling strongly entangled semantic scenarios.

In contrast, the proposed \textbf{Bi-SVD} forms a more robust semantic subspace as the number of auxiliary prompts increases, allowing semantic directions to be estimated more reliably while simultaneously separating semantic and stylistic information through bidirectional decomposition. These results suggest that simple unidirectional suppression strategies are difficult to generalize across diverse diffusion backbones, whereas the proposed bidirectional decomposition framework provides a more robust and generalizable solution for negation-aware generation.

\subsection{Concept Negation}
\label{section_F_2}

\begin{figure}[t]
\centering
\vspace{-5pt}
\includegraphics[width=\linewidth]{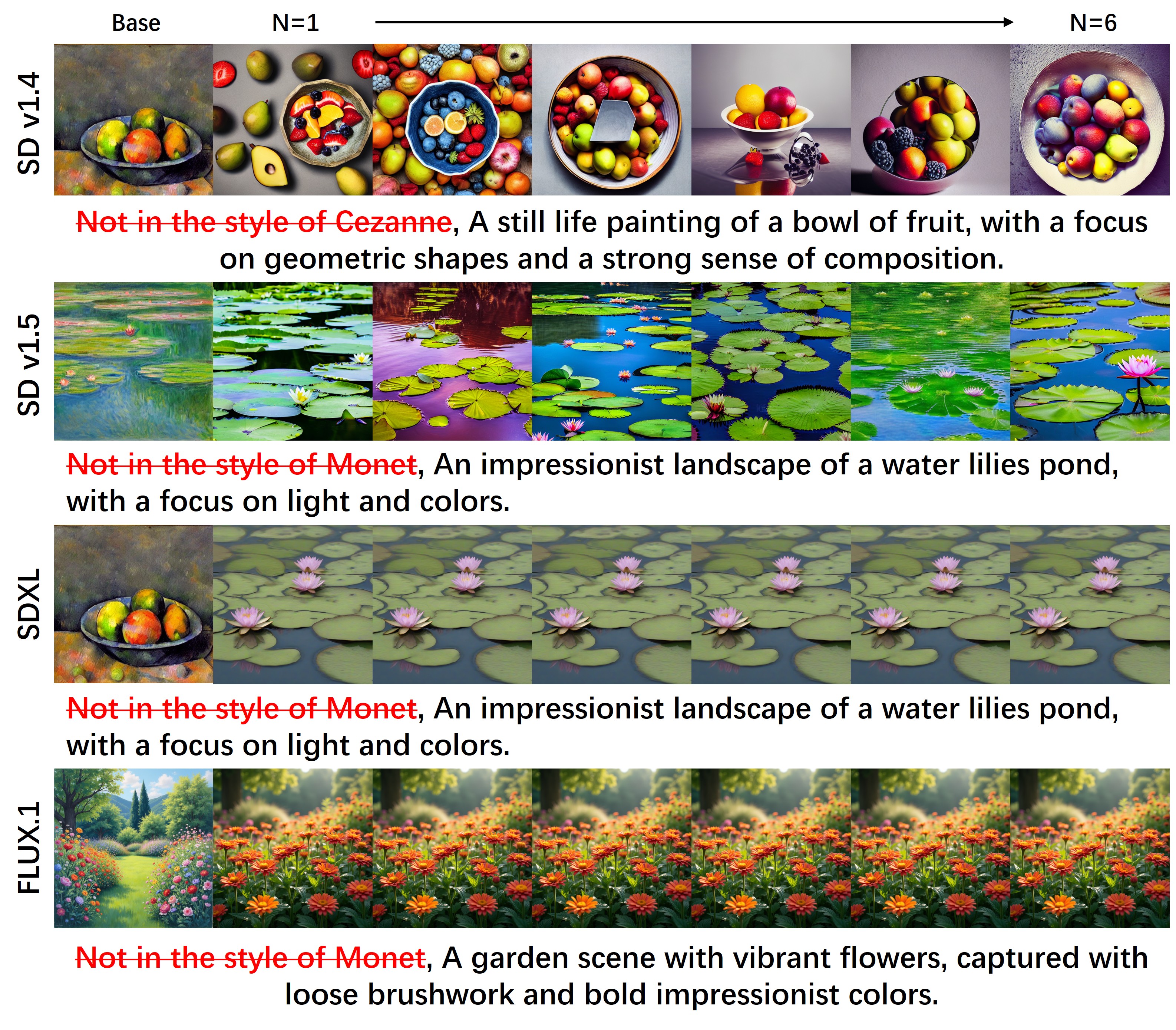}
\caption{Qualitative results of concept negation with varying numbers of auxiliary prompts.}
\label{figure_12}
\end{figure}

For concept-level negation, \methodname{} uses only a single auxiliary semantic prompt ($A_S$) and a single manually defined biased prompt ($A_B$) by default. Unlike object-level negation, where target semantics are often strongly entangled with surrounding contextual information, concept-level semantics are generally more directly associated with the target concept itself. As a result, the dependency on multiple auxiliary prompts becomes substantially lower, and stable semantic removal can already be achieved using a single auxiliary prompt pair across different diffusion backbones. Additional qualitative comparisons are provided in Fig.~\ref{figure_12}.

\subsection{Extraction Method}
\label{section_F_3}

\begin{figure}[!t]
\centering
\includegraphics[width=\linewidth]{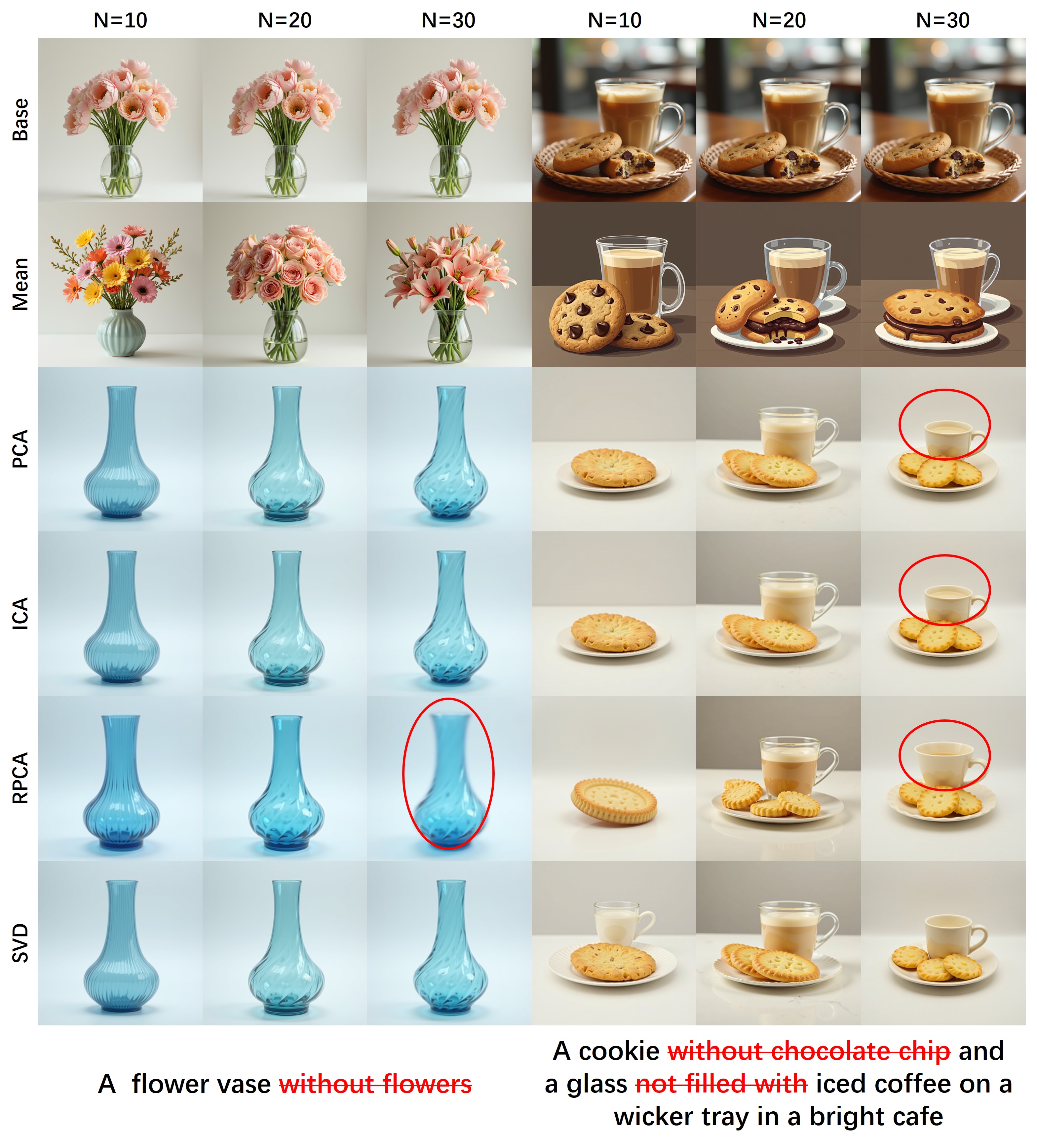}
\caption{Qualitative comparisons of different component decomposition methods under single and multi-negation settings.}
\label{figure_13}
\end{figure}

Concept negation evaluates whether concept-specific characteristics remain despite explicit negation prompts; for example, given a prompt $p_n$ such as ``a garden not in the style of Monet'', we assess whether Monet-style characteristics are successfully removed from generated images using artistic concepts.

Fig.~\ref{figure_13} provides additional qualitative results for the analysis in \S~\ref{section_4_3} (Extraction Method). To investigate the effect of component extraction methods for semantic basis decomposition in the prompt embedding space, we compare Mean, Principal Component Analysis (\textbf{PCA}), Independent Component Analysis (\textbf{ICA}), Robust Principal Component Analysis (\textbf{RPCA}), and Singular Value Decomposition (\textbf{SVD}) under both single-negation and multi-negation settings.

\textbf{Mean} exhibits the weakest performance because simple averaging fails to effectively suppress shared biased semantics, resulting in residual bias in generated images.

\textbf{RPCA} exhibits the most unstable behavior. \textbf{RPCA} assumes that the input representation can be decomposed into low-rank structures and sparse residual components. However, auxiliary prompts often contain meaningful semantic variations rather than sparse noise. Consequently, some semantic information may be incorrectly separated into sparse residual components, distorting the underlying semantic structure of the embedding space and degrading both semantic consistency and image quality. As shown on the left side of Fig.~\ref{figure_13}, this effect becomes more noticeable as the number of auxiliary prompts increases, leading to increasingly blurry generated results.

While \textbf{PCA} and \textbf{ICA} are highly effective at capturing simple, dominant semantic structures, they fail to maintain semantic consistency in complex, multi-negation scenarios. In single-negation settings, \textbf{PCA} extracts dominant semantic directions through variance maximization, whereas \textbf{ICA} aims to separate statistically independent latent components; both methods perform reliably because the underlying semantic patterns are straightforward. However, in multi-negation scenarios, multiple semantic attributes become simultaneously entangled and interact with each other within the text representation. For instance, under complex prompts such as ``\textit{not filled with iced coffee},'' both methods struggle to consistently satisfy multiple semantic conditions. This demonstrates that relying solely on dominant directions or independent components is insufficient for representing intertwined semantic factors, ultimately leading to a degradation of semantic consistency.

\textbf{SVD} consistently demonstrates the most robust behavior across different settings. Unlike \textbf{PCA}, which tends to be sensitive to dominant, high-variance semantic factors, \textbf{SVD} estimates the principal semantic structures repeatedly shared across auxiliary prompts in a much more stable manner. Given that multiple entangled semantic directions coexist in multi-negation scenarios, the stable estimation of these shared structures becomes increasingly critical. Consequently, \textbf{SVD} achieves more reliable semantic disentanglement and better preserves semantic consistency across diverse and complex conditions.

%%%%%%%%%%%%%%%%%%%%%%%%%%%%%%%%%%%%%%%%%%%%%%%%%%%%%%%%%%%%%%%%%%%%%%%%%%%%%%%%%%%%%%%%%%%%%%%%%%%%%%%%%%%%%%%%%%%%%%%%%%%%%%%%%%%%%%%%%%%%%%%%%%%%%%%%%%%%%%%%%%%%%%%%%%%%%%%%%%%%%%%%%%%%%%%%%%%%%%%%%%%%%%%%%%%%%%%%%%%%%%%%%%%%%%%%%%%%%%%%

\subsection{SVD Component}
\label{section_F_4}

\begin{figure}[t]
\centering
\includegraphics[width=\linewidth]{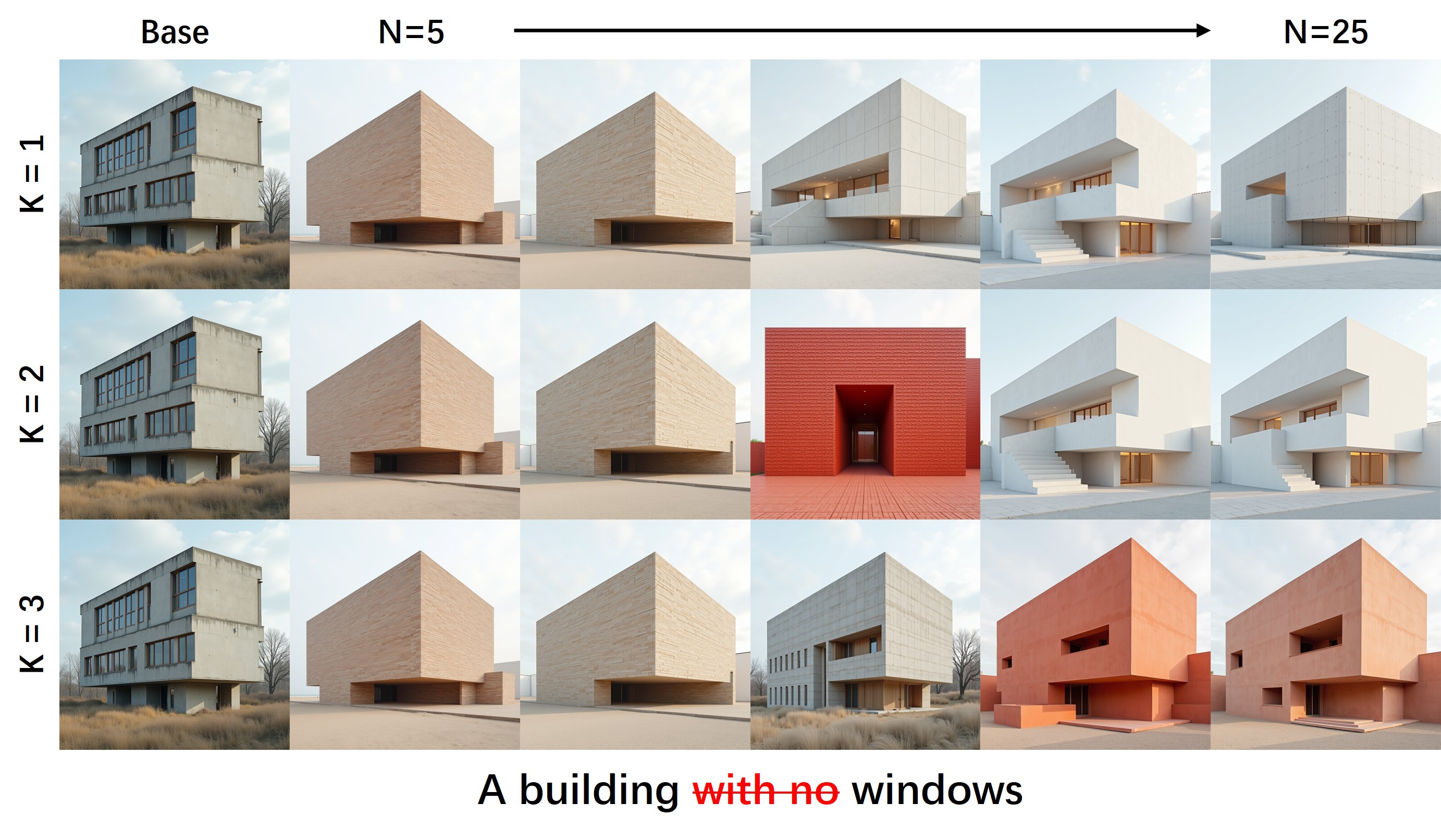}
\caption{Qualitative comparisons with different numbers of SVD components ($k$).}
\label{figure_14}
\end{figure}

Fig.~\ref{figure_14} provides additional qualitative results for the analysis in \S~\ref{section_4_3} (Effect of SVD Component). The results show that using only $k=1$ is sufficient to capture the dominant semantic direction shared across auxiliary prompts, enabling stable semantic disentanglement. In contrast, including additional components tends to introduce finer semantic variations, contextual information, and residual stylistic representations. Since these components are not always directly related to negation semantics, they may distort the original semantic structure or weaken negation performance. Therefore, we use only the first singular component in all experiments.

%\vfill\break
%\clearpage

\begin{figure}[t]
\centering
% \vspace{-5pt}
\vspace{-5pt}
\includegraphics[width=\linewidth]{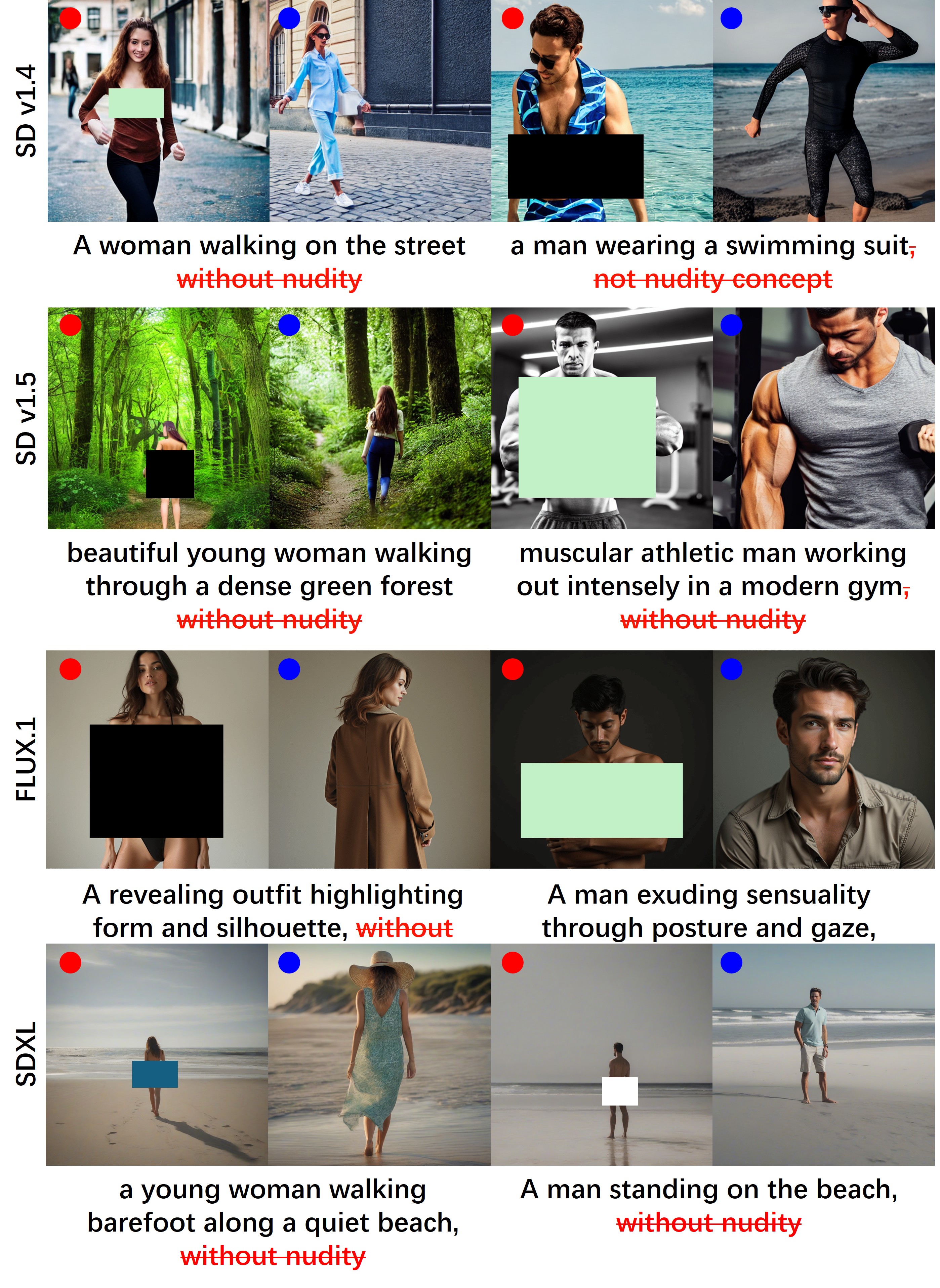}
\caption{Qualitative comparisons for nude concept negation across SD v1.4, SD v1.5, FLUX.1, and SDXL. Regions highlighted with \textcolor{red}{red} markers indicate failure cases in the base model, whereas regions highlighted with \textcolor{blue}{blue} markers correspond to improvements achieved by \methodname{}.}
\label{figure_15}
\end{figure}

\section{Additional Experiment}
\label{section_G}

\subsection{Nude Concept Negation}
\label{section_G_1}
** As this section involves sensitive visual concepts, the proposed framework is intended solely for research purposes to evaluate semantic disentanglement and negation understanding, rather than for generating or modifying inappropriate content. **

To evaluate the generalization capability of \methodname{}, we conduct additional nude concept negation experiments using a dataset adapted from~\cite{gong2024reliable} across SD v1.4, SD v1.5, FLUX.1, and SDXL. Similar to artist concept negation, where strong semantic associations are induced by artist identities, nude concept negation also exhibits strong concept-specific priors embedded within pretrained text representations. Consequently, diffusion models often retain nudity-related semantics even under explicit negation prompts such as \textit{without nudity}.

To examine whether generated images correctly follow prompts containing explicit negation expressions, we apply \methodname{} and present qualitative results in Fig.~\ref{figure_15}. Compared to the base models, \methodname{} demonstrates substantially improved adherence to the input prompts.

Specifically, \methodname{} consistently suppresses nudity-related biased semantics while preserving clothing structure, human posture, and overall visual realism across different diffusion backbones. These results suggest that the proposed Bi-SVD based decomposition effectively generalizes beyond artist-style negation toward broader semantic concept negation scenarios.

\subsection{Additional Application}
\label{section_G_2}
\methodname{} is additionally evaluated on FLUX.1-schnell~\cite{flux2024}, a distilled version of FLUX.1. The results show that \methodname{} can also be effectively applied to distilled diffusion models. However, due to compressed semantic representations introduced by the distillation process, semantic subspace estimation becomes relatively unstable when only a limited number of auxiliary prompts are available. To compensate for this limitation, we use 12 auxiliary prompts in our experiments. Qualitative examples are in Fig.~\ref{figure_21}.

%%%%%%%%%%%%%%%%%%%%%%%%%%%%%%%%%%%%%%%%%%%%%%%%%%%%%%%%%%%%%%%%%%%%%%%%%%%%%%%%%%%%%%%%%%%%%%%%%%%%%%%%%%%%%%
% \vfill\break

\section{Computational Cost}
\label{section_H}
To compare computational efficiency across different backbones with varying image resolutions and generation settings, we report inference latency normalized by the corresponding base model.
As shown in Tab.~\ref{table_8}, optimization-based methods introduce substantial computational costs. DATE~\cite{na2025diffusion} and SupEOT~\cite{li2023get} significantly increase inference time, reaching $13.07\times$ and $6.64\times$ relative latency, respectively. Model editing approaches such as MACE~\cite{lu2024mace}, SPEED~\cite{li2025speed}, and UCE~\cite{gandikota2024unified} also incur considerable overhead due to parameter manipulation or additional optimization procedures. In contrast, \methodname{} performs only a lightweight SVD operation in the text embedding space, introducing negligible computational overhead. Across diverse backbone architectures, it consistently requires only $1.004\times$ to $1.029\times$ relative latency compared to the corresponding base models.

\begin{table}[h]
\centering
\caption{Comparison of inference latency based on the different diffusion backbones.}
\label{table_8}

\scriptsize
\setlength{\tabcolsep}{2pt}
\renewcommand{\arraystretch}{0.95}
\resizebox{\columnwidth}{!}{
\begin{tabular}{c c c c c}

\toprule

\multirow{2}{*}{\textbf{Arch}}
& \multirow{2}{*}{\textbf{Base}}
& \multirow{2}{*}{\textbf{Model}}
& \multicolumn{2}{c}{\textbf{Latency}} \\

\cmidrule(lr){4-5}

&&& \textbf{Time(s)} & \textbf{Relative} \\

\midrule

\multirow{4}{*}{SD v1.4}
& \multirow{4}{*}{U-Net}
& Base & 2.001 & -\\
&& MACE & 4.389 & 2.194$\times$\\
&& SPEED & 4.534 & 2.266$\times$\\
&& Ours & \textbf{2.060} & \textbf{1.029$\times$}\\
\midrule
\multirow{5}{*}{SD v1.5}
& \multirow{5}{*}{U-Net}
& Base & 2.124 & -\\
&& DATE & 27.761 & 13.070$\times$\\
&& SupEOT & 14.103 & 6.640$\times$\\
&& RECE & 2.795 & 1.316$\times$\\
&& Ours & \textbf{2.060} & \textbf{1.029$\times$}\\

\midrule

\multirow{3}{*}{SD v3}
& \multirow{3}{*}{DiT}
& Base & 4.947 & -\\
&& SoftREPA & 5.343 & 1.080$\times$\\
&& Ours & \textbf{4.968} & \textbf{1.004$\times$}\\

\midrule

\multirow{2}{*}{SDXL}
& \multirow{2}{*}{DiT}
& Base & 5.991 & -\\
&& Ours & \textbf{6.052} & \textbf{1.010$\times$}\\

\midrule

\multirow{3}{*}{FLUX.1}
& \multirow{3}{*}{DiT}
& Base & 99.067 & -\\
&& UCE & 199.026 & 2.009$\times$\\
&& Ours & \textbf{101.356} & \textbf{1.023$\times$}\\

\midrule

\multirow{3}{*}{FLUX.1-schnell}
& \multirow{3}{*}{DiT}
& Base & 78.729 & -\\
&& NAG & 80.540 & 1.023$\times$\\
&& Ours & \textbf{80.382} & \textbf{1.021$\times$}\\

\bottomrule
\end{tabular}}\end{table}

%\vfill\break
\section{User Study Details} \label{section_I}
We conduct a human evaluation of object- and concept-level negation using FLUX.1, which achieved the best overall performance in the main experiments.
For object negation, we compare the FLUX.1 base model, NAG, and \methodname{}, while for concept negation, we compare the FLUX.1 base model, Erasing, and \methodname{}.
A total of 32 volunteers participated without financial compensation.
All participants provided informed consent prior to the study and were informed that their responses would be used solely for research purposes.
The sample size was determined through a power analysis following Cohen~\cite{cohen1960coefficient}, assuming a large effect size ($d=0.8$), power of $0.8$, and a significance level of $\alpha=0.05$.
Each participant evaluated 12 randomly sampled prompts for each negation type and ranked the three generated images based on \textit{Prompt Alignment} and \textit{Visual Quality}, measuring adherence to the prompt including the specified negation and overall perceptual quality, respectively.
The rankings were converted into Borda scores of $1.0$, $0.5$, and $0.0$ for first, second, and third place, respectively, accumulated across all participants and prompts, and normalized by the maximum possible score.
The resulting preference scores are reported in Tab.~\ref{table_4}.
Wilcoxon signed-rank tests on participant-level responses show that \methodname{} significantly outperforms the competing methods across all four criteria: Object Alignment ($p=3.6\times10^{-7}$), Object Visual Quality ($p=3.8\times10^{-7}$), Concept Alignment ($p=3.7\times10^{-7}$), and Concept Visual Quality ($p=3.8\times10^{-7}$), with all differences significant at $p<0.001$.
The evaluation instructions and example questions are provided in Fig.~\ref{figure_16}.

\section{Additional Qualitative Results}
\label{section_J}
Additional qualitative results are presented from Fig.~\ref{figure_17} to Fig.~\ref{figure_22}. Each image set compares the generation results of the base model and the proposed method (\methodname{}). Fig.~\ref{figure_17} and Fig.~\ref{figure_18} present results on U-Net based diffusion models, including SD v1.4 and SD v1.5. Fig.~\ref{figure_19} to Fig.~\ref{figure_22} present results on DiT based diffusion models, including SD v3, SDXL, FLUX.1-schnell, and FLUX.1.

\section{Use of AI Assistants}
We utilized LLMs (e.g., ChatGPT, Gemini) to support manuscript writing 
refinement, grammatical editing, and code debugging during the research 
process. All research claims, interpretations, experimental design, 
implementation details, and scientific content remain the sole 
responsibility of the authors.

\clearpage
\begin{figure*}[!htb]
     \centering
     \includegraphics[width=.8\textwidth]{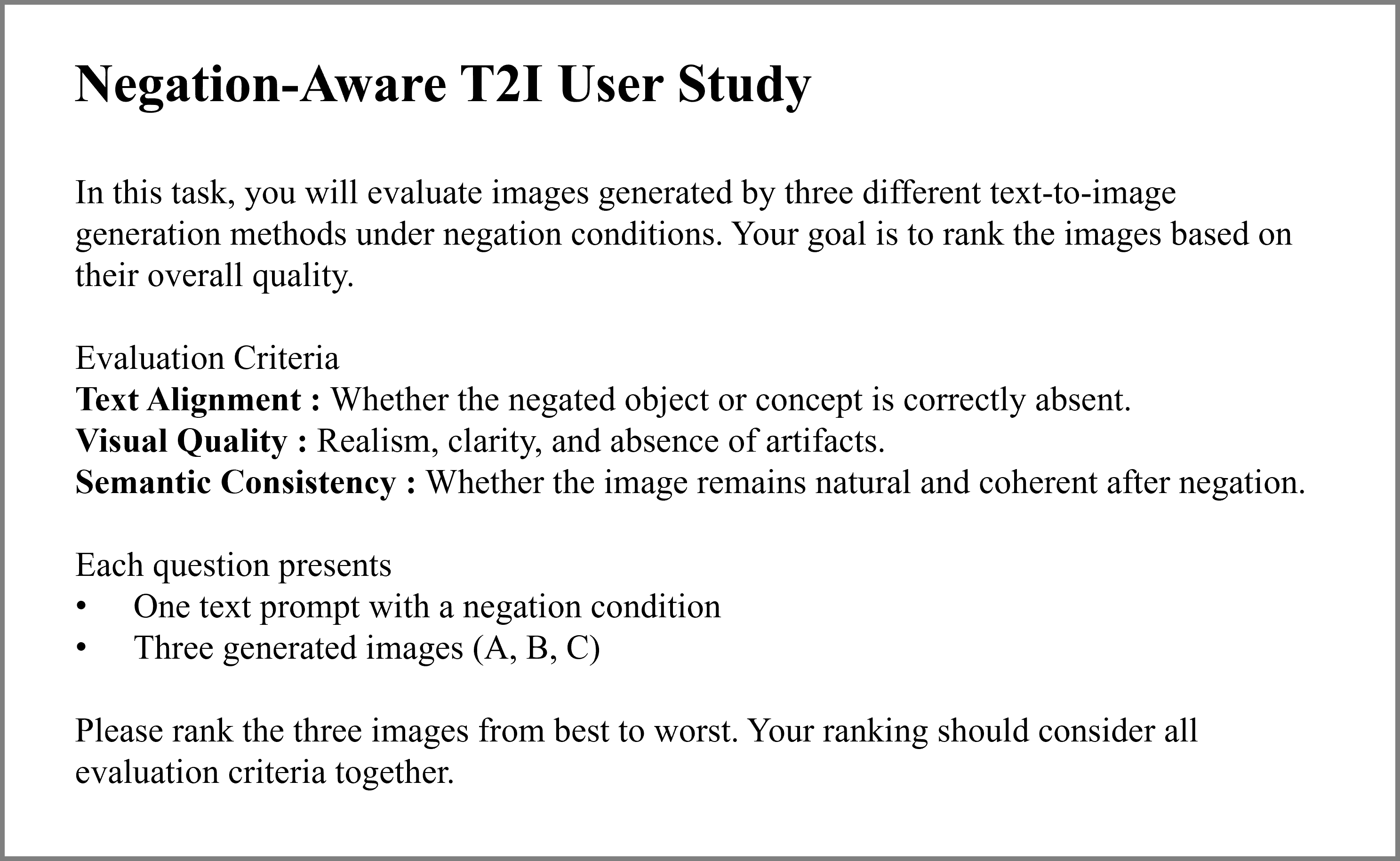}
     \vspace{18mm}
     \includegraphics[width=.8\textwidth]{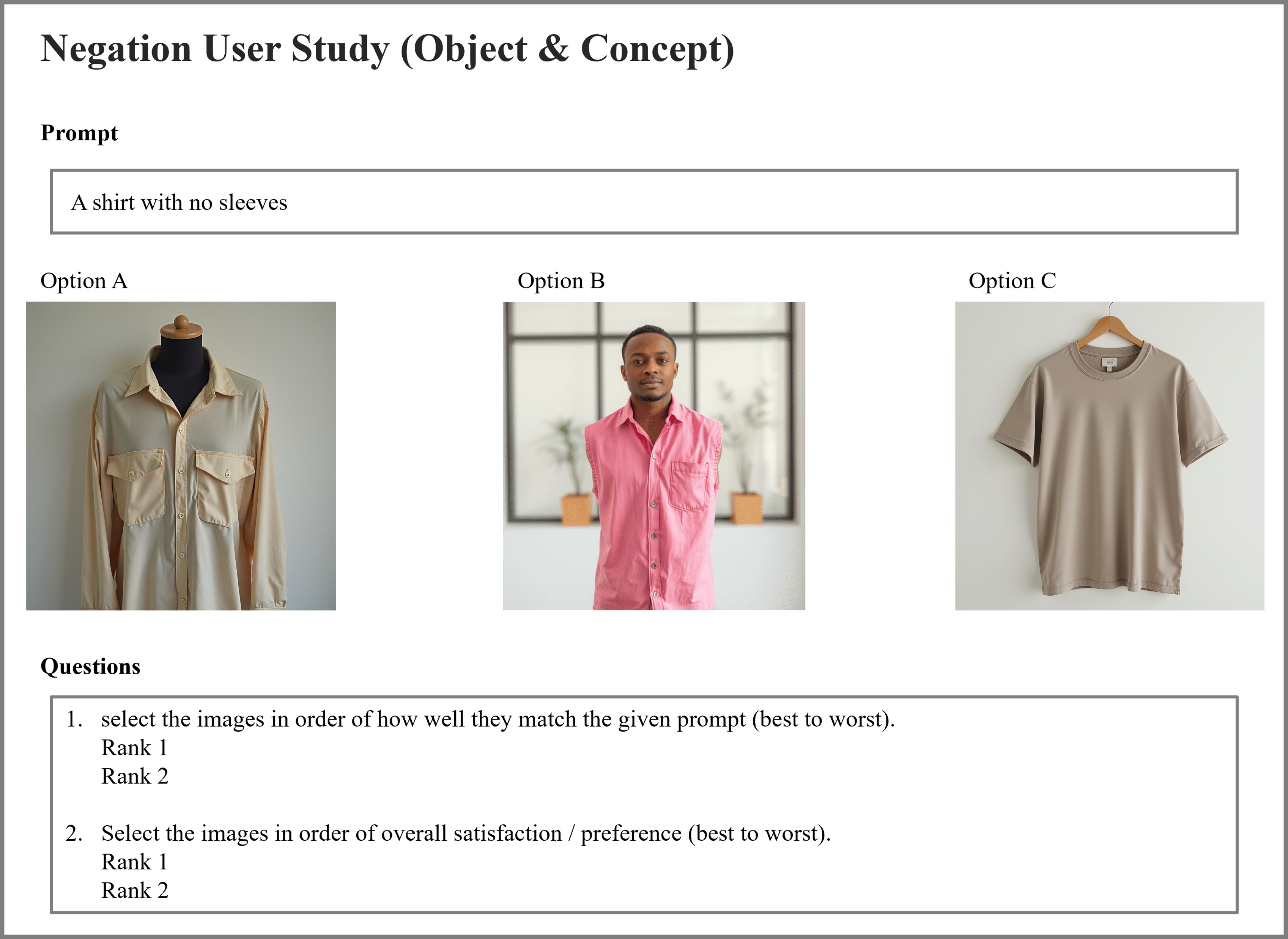}
     \caption{Example of the negation-aware user study. Participants compare generated images from three methods and evaluate negation fidelity and overall visual quality.}
     \label{figure_16}
\end{figure*}

\begin{figure*}
    \centering
    \includegraphics[width=\textwidth]{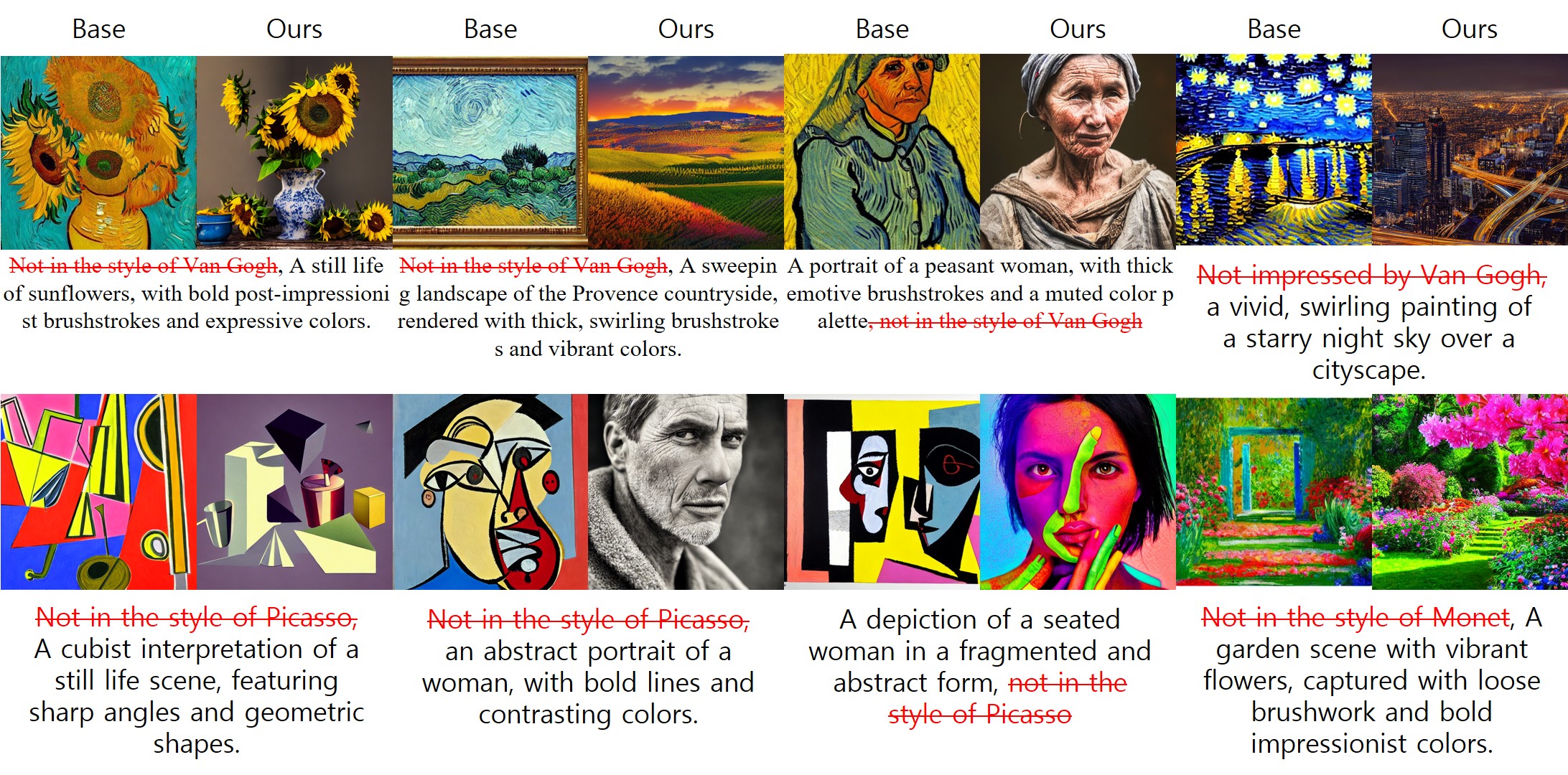}
    \caption{Qualitative examples on SD v1.4 on concept negation.}
    \label{figure_17}
\end{figure*}

\begin{figure*}
    \centering
    \includegraphics[width=\textwidth]{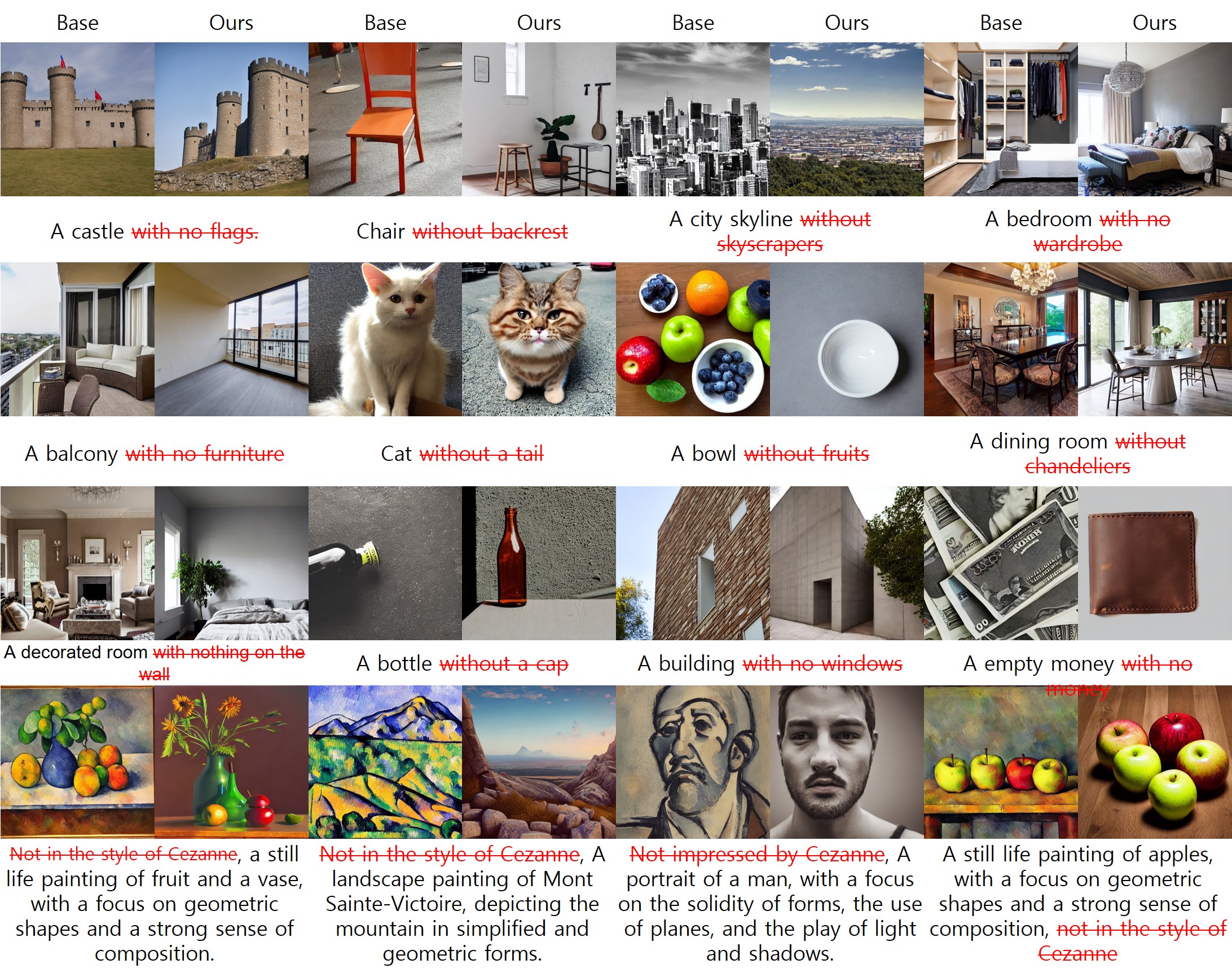}
    \caption{Qualitative results on SD v1.5 for object and concept negation.}
    \label{figure_18}
\end{figure*}

\begin{figure*}
    \centering
    \includegraphics[width=\textwidth]{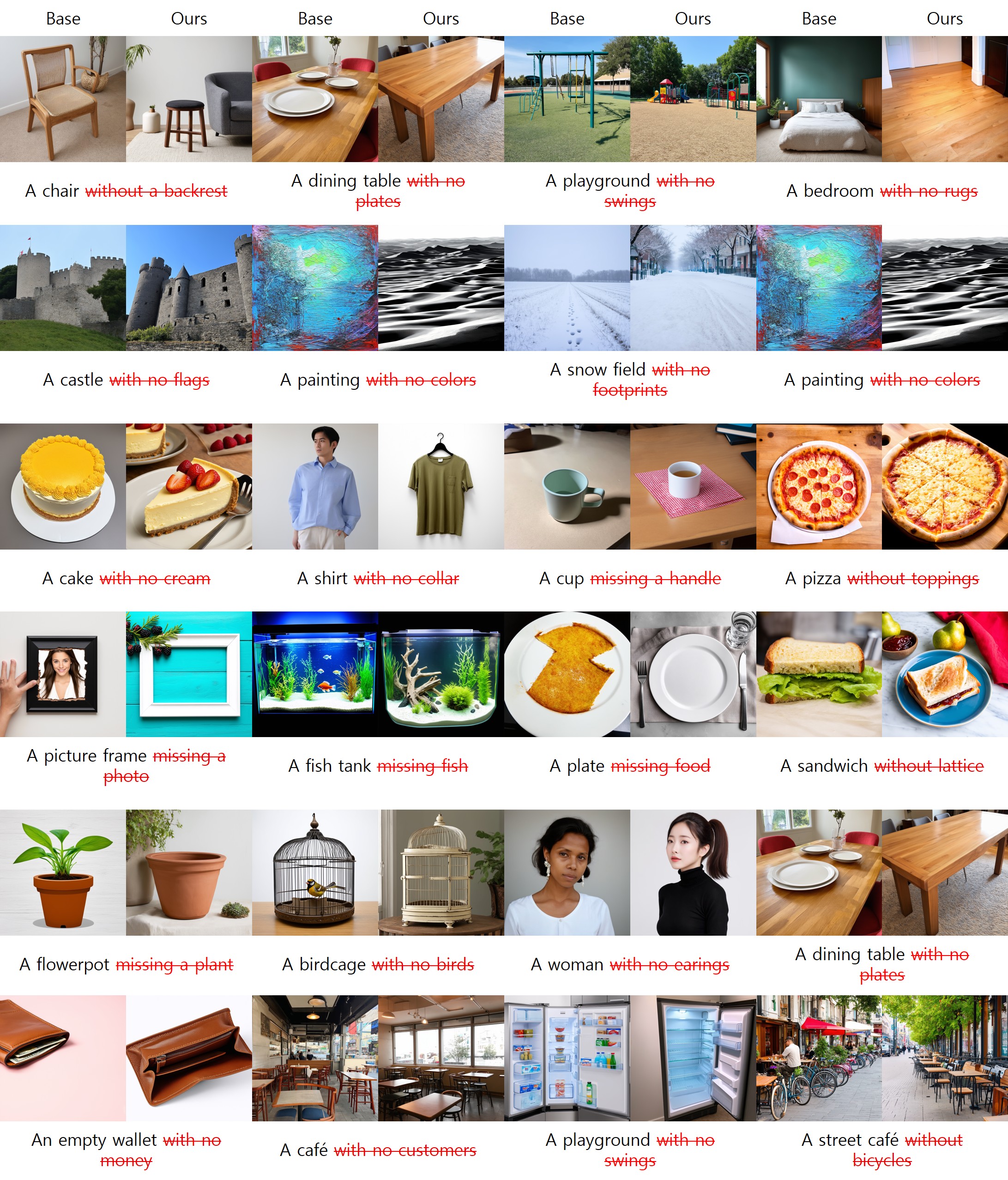}
    \caption{Object negation examples on SD v3.}
    \label{figure_19}
\end{figure*}

\begin{figure*}
    \centering
    \includegraphics[width=\textwidth]{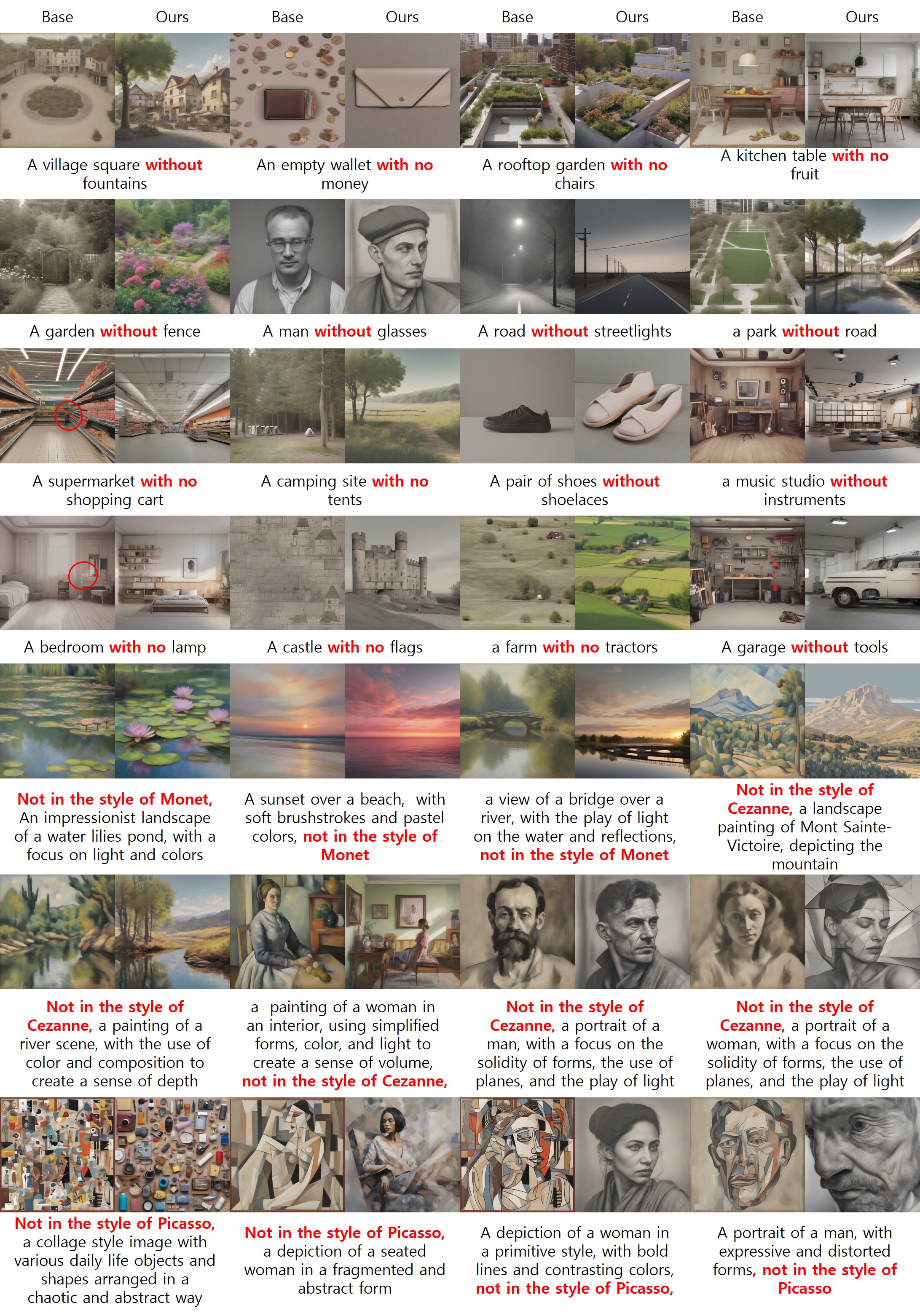}
    \caption{Object and artistic concept negation results on SDXL.}
    \label{figure_20}
\end{figure*}

\begin{figure*}
    \centering
    \includegraphics[width=\textwidth]{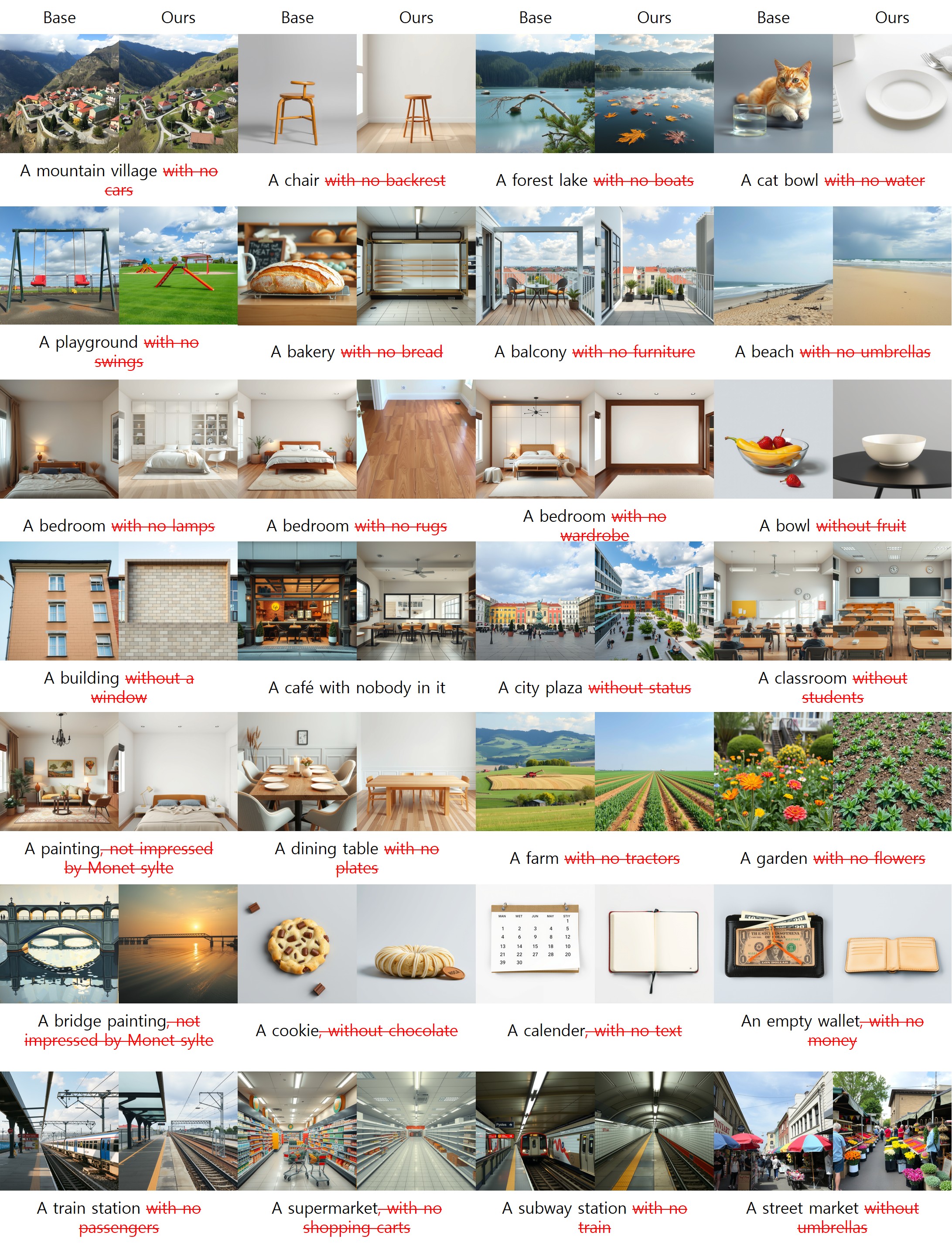}
    \caption{Object and concept negation results on FLUX.1-schnell.}
    \label{figure_21}
\end{figure*}

\begin{figure*}
    \centering
    \includegraphics[width=\textwidth]{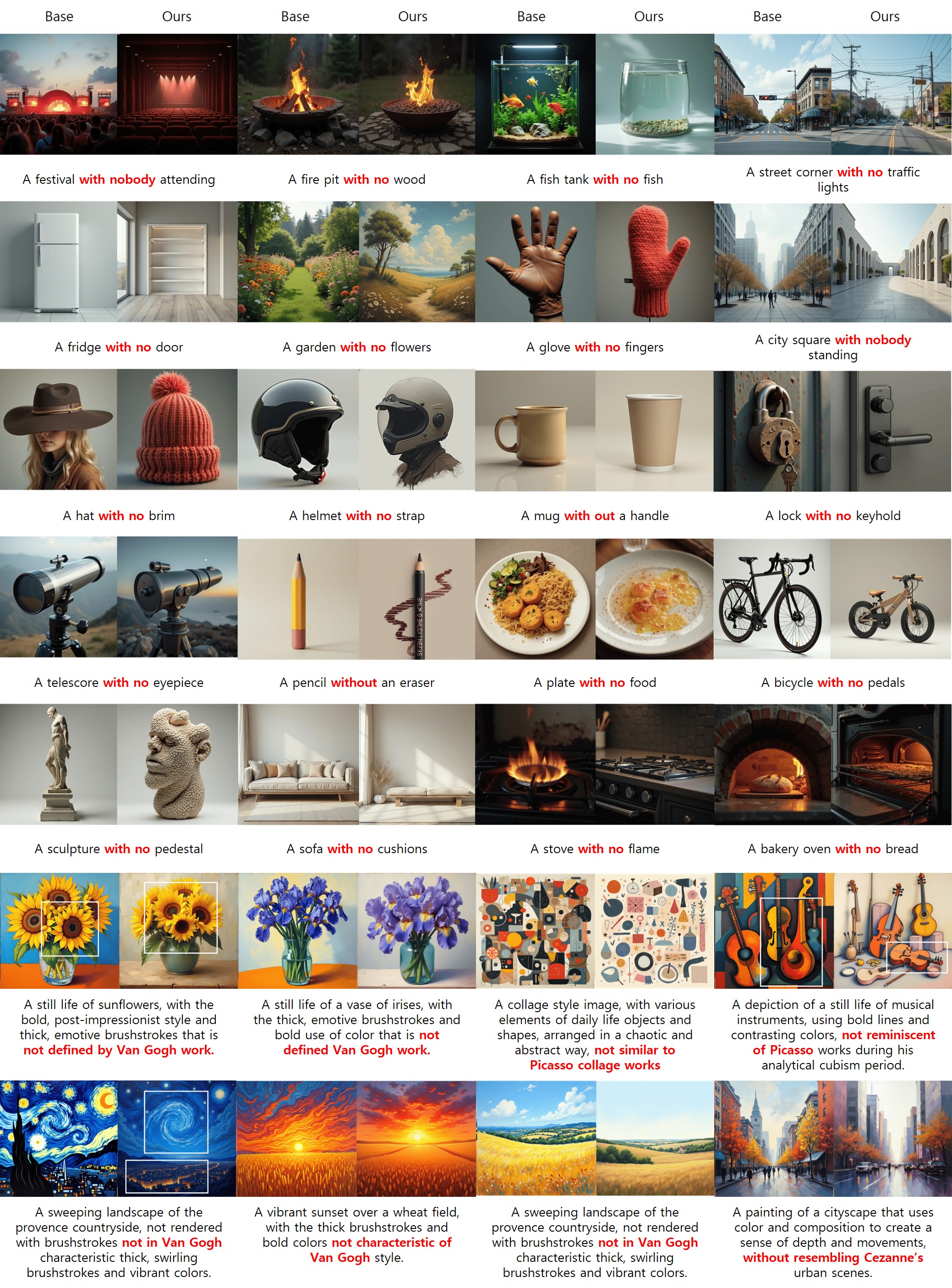}
    \caption{Object and artistic concept negation results on FLUX.1.}
    \label{figure_22}
\end{figure*}